\documentclass[11pt]{article}

\usepackage[letterpaper,margin=1in]{geometry}
\usepackage[T1]{fontenc}
\usepackage{lmodern}
\usepackage{silence}
\usepackage{microtype}
\usepackage[authoryear,round]{natbib}
\setcitestyle{authoryear,round,citesep={;},aysep={,},yysep={;}}

\usepackage{amsmath,amssymb,amsthm,amsfonts,bm,mathtools}
\usepackage{mathrsfs}
\usepackage{enumitem}
\usepackage{graphicx}
\usepackage{xcolor}
\usepackage{float}
\usepackage{placeins}
\usepackage{array}
\usepackage{needspace}
\usepackage{etoolbox}
\usepackage{url}
\usepackage{hyperref}
\usepackage[nameinlink,capitalise,noabbrev]{cleveref}

\setlist{topsep=4pt,itemsep=2pt,parsep=0pt,partopsep=0pt}
\allowdisplaybreaks[4]
\DeclareMathOperator{\ra}{ran}
\DeclareMathOperator{\Tr}{Tr}

\DeclareMathOperator*{\Span}{span}

\newcommand{\RE}{\mb{E}_{\mm{Rad}^{mn}}}

\newcommand{\eR}{\widehat{\mf{R}}_n^m}
\newcommand{\W}{\msf{W}(C,D)}
\newcommand{\Wl}{\msf{W}_l(C,D)}
\newcommand{\HB}{H^{\scriptscriptstyle (\mm{B})}}
\newcommand{\bsigma}{\bm{\sigma}}

\newcommand{\mK}{\m{K}}
\newcommand{\N}{\mb{N}}
\newcommand{\R}{\mb{R}}
\newcommand{\Rnn}{\R_{\ge 0}}
\newcommand{\Zp}{\mb{N}_{>0}}

\newcommand{\bea}{\b e_a}
\newcommand{\beb}{\b e_b}
\newcommand{\bomega}{\bm{\omega}}
\newcommand{\kB}{k^{\scriptscriptstyle (\mm{B})}}

\newcommand{\mD}{\m{D}}
\newcommand{\mY}{\m{Y}}

\newcommand{\mX}{\m{X}}
\newcommand{\mC}{\m{C}}
\newcommand{\F}{\m{F}}
\newcommand{\mH}{\m{H}}

\newcommand{\mS}{\m{S}}
\newcommand{\mT}{\m{T}}
\newcommand{\mF}{\m{F}}
\newcommand{\mG}{\m{G}}
\newcommand{\mV}{\m{V}}
\newcommand{\mI}{\m{I}}

\newcommand{\mf}[1]{\mathfrak{#1}}
\newcommand{\msf}[1]{\mathsf{#1}}
\newcommand{\mb}[1]{\mathbb{#1}}
\newcommand{\m}[1]{\mathcal{#1}}

\newcommand{\mm}[1]{\mathrm{#1}}

\renewcommand{\d}{\mathrm{d}}
\renewcommand{\b}{\mathbf}
\renewcommand{\P}{\mathbb{P}}

\newcommand{\proofstep}[1]{\par\smallskip\noindent\emph{#1}\enspace}
\newcommand{\HMs}[1]{H_{\b M}^{#1}}

\newcommand{\mE}{\m{E}}
\newcommand{\mL}{\m{L}}
\newcommand{\mZ}{\m{Z}}

\newcommand{\PaperTitle}{Operator-Theoretic Generalization Bounds for Multitask Deep Learning}
\newcommand{\PaperPDFAuthors}{Mahdi Mohammadigohari, Thomas Borsani, and Giuseppe Di Fatta}
\hypersetup{
  colorlinks=true,
  linkcolor=blue,
  citecolor=blue,
  urlcolor=blue,
  pdftitle={\PaperTitle},
  pdfauthor={\PaperPDFAuthors},
  pdfpagemode=UseNone
}

\newtheorem{theorem}{Theorem}
\newtheorem{lemma}{Lemma}[section]
\newtheorem{proposition}{Proposition}
\newtheorem{corollary}{Corollary}
\theoremstyle{definition}
\newtheorem{assumption}{Assumption}
\newtheorem{definition}{Definition}
\theoremstyle{remark}
\newtheorem{remark}{Remark}
\theoremstyle{plain}

\crefname{assumption}{Assumption}{Assumptions}
\Crefname{assumption}{Assumption}{Assumptions}
\crefname{theorem}{Theorem}{Theorems}
\Crefname{theorem}{Theorem}{Theorems}
\crefname{lemma}{Lemma}{Lemmas}
\Crefname{lemma}{Lemma}{Lemmas}
\crefname{proposition}{Proposition}{Propositions}
\Crefname{proposition}{Proposition}{Propositions}
\crefname{corollary}{Corollary}{Corollaries}
\Crefname{corollary}{Corollary}{Corollaries}
\crefname{remark}{Remark}{Remarks}
\Crefname{remark}{Remark}{Remarks}

\title{\PaperTitle}
\author{
Mahdi Mohammadigohari\\
Free University of Bozen--Bolzano, Italy\\
\texttt{mahdi.mohammadigohari@gmail.com}
\and
Thomas Borsani\\
Free University of Bozen--Bolzano, Italy\\
\texttt{tborsani@unibz.it}
\and
Giuseppe Di Fatta\\
Free University of Bozen--Bolzano, Italy\\
\texttt{giuseppe.difatta@unibz.it}
}
\date{}

\begin{document}

\maketitle

\begin{abstract}
We develop operator-theoretic generalization bounds for deep multi-output function classes by representing network layers as Koopman composition operators on vector-valued reproducing kernel Hilbert spaces. In vector-valued Sobolev RKHSs, we derive Rademacher complexity bounds for invertible and width-expanding injective architectures. The estimates separate the output-coupling contribution, represented by the trace of the task matrix, from the layerwise operator norms, Sobolev symbol ratios, determinant factors, and restriction constants generated by the linear maps. We then analyze a distinct one-dimensional Brownian/Cameron--Martin regime. Using the exact anchored derivative-norm characterization of the vector-valued Brownian RKHS, we obtain layerwise bounds for domain-preserving scalar linear maps and anchored diffeomorphic activations; the corresponding factors scale as $|W_l|^{1/2}$ and $\|\sigma_l'\|_\infty^{1/2}$, respectively, and do not involve Sobolev smoothness exponents. Because the Sobolev and Brownian results concern different hypothesis spaces, neither is asserted to dominate the other uniformly. We additionally formulate shared operator learning across tasks, prove a finite-rank representer theorem, derive the exact finite-dimensional problem for squared loss, and establish a target-transfer bound when the learned operator is obtained independently of the target sample. Synthetic and MNIST studies examine stabilized Sobolev-inspired and Brownian-inspired complexity proxies; these empirical proxies are not evaluations of the proved bounds for rank-deficient architectures.
\end{abstract}

\section{Introduction}\label{sec:intro}

Understanding the generalization behavior of deep networks remains a central problem in learning theory. Classical approaches control capacity through parameter counts, margins, compression, or norms of the weight matrices~\citep{bartlett2002rademacher,mohri2018foundations,Neyshabur2015,golowich2018size,bartlett2017spectrally,wei2019data,Fanghui2024,arora2018compression}. Norm-based estimates can avoid an explicit width dependence, but their layerwise factors often retain only coarse information about the geometry of the transformations.

A complementary approach represents each layer by the Koopman operator that it induces on a function space. The closest prior generalization analysis in this direction develops scalar-valued bounds in Sobolev RKHSs and exposes determinant-based volume distortion for full-rank layers~\citep{hashimoto2024koopmanbased}. Two questions then arise naturally. First, how does this construction change for multi-output classes in vector-valued RKHSs? Second, what changes when the underlying function space is not Sobolev but an anchored Brownian/Cameron--Martin RKHS?

\proofstep{Function-space viewpoint.}
The paper studies both questions at the level of the induced composition operators. In the Sobolev regime, the output coordinates are coupled by a positive-definite task matrix $\b M$, and the corresponding RKHS norm is the $\b M^{-1}$-weighted vector Sobolev norm. This makes it possible to separate the task-coupling factor $\Tr\left(\b M\right)$ from the layerwise geometric terms. We treat square invertible maps and non-square injective maps; in the latter case, the proof also records the Sobolev restriction cost created by the lower-dimensional range of each layer.

\proofstep{Brownian/Cameron--Martin viewpoint.}
For the one-dimensional Brownian kernel, the RKHS norm is exactly an anchored first-derivative energy. This permits a direct change-of-variables analysis without Fourier extensions. Under the domain-preserving conditions needed to keep every composition inside the RKHS, a scalar linear layer contributes $|W|^{1/2}$ and an anchored $C^1$ diffeomorphism contributes $\|\sigma'\|_\infty^{1/2}$. The Brownian and Sobolev bounds therefore reflect different function-space geometries. They are bounds for different hypothesis classes and are not presented as uniformly comparable numerical estimates.

\paragraph{Contributions.}
The principal contributions are as follows.
\begin{itemize}[leftmargin=2em]
\item \textbf{Vector-valued Sobolev Koopman bounds.} We extend the Koopman factorization to separable vector-valued Sobolev RKHSs and derive multi-output Rademacher bounds. The output matrix enters through $\Tr\left(\b M\right)$ and the $\b M^{-1}$-weighted terminal norm, while the linear layers contribute operator-norm, Sobolev-symbol, and determinant factors.

\item \textbf{Injective width expansion.} For injective non-square layers, we state explicit Sobolev trace conditions, define the inverse transpose on the range, and include uniform restriction-operator constants. The rank-deficient case is retained only as a regularized empirical proxy; no bounded Koopman theorem is claimed there.

\item \textbf{One-dimensional Brownian/Cameron--Martin bound.} We give an exact vector-valued Cameron--Martin characterization and prove the linear and activation composition estimates directly from the derivative norm. This yields a complexity bound without Sobolev smoothness exponents.

\item \textbf{Shared operator learning and transfer.} We prove a finite-rank representer theorem for a shared Hilbert--Schmidt operator, derive the exact coefficient problem for squared loss, and establish a conditional target bound under source--target sample independence.

\item \textbf{Empirical proxy study.} Synthetic and MNIST experiments compare stabilized Sobolev-inspired and Brownian-inspired factors. The text distinguishes these proxies from the proved bounds, especially for rectangular or rank-deficient layers.
\end{itemize}

\section{Related Work}\label{sec2}

\paragraph{Norm-based generalization bounds.}
A large literature controls deep-network capacity through spectral, Frobenius, path, or reference-matrix norms~\citep{Neyshabur2015,golowich2018size,bartlett2017spectrally,wei2020improved,Ju2022Robust}. These analyses can avoid direct parameter counting, but their bounds generally summarize each layer through a limited collection of matrix norms. Compression methods can exploit approximate low-rank structure~\citep{arora2018compression}, while orthogonality-based analyses emphasize conditioning and dynamical stability~\citep{li2021orthogonal}. The operator-theoretic approach considered here is different: it controls the action of each layer on a specified function space and thereby exposes the volume and regularity distortion of the induced composition operator.

\paragraph{Koopman formulations.}
Koopman operators encode nonlinear dynamics through linear composition operators on observables. A recent learning-theoretic formulation applies this perspective to scalar deep networks in Sobolev RKHSs and derives determinant-sensitive generalization bounds for full-rank weights~\citep{hashimoto2024koopmanbased}. Our Sobolev results extend that construction to multi-output vector-valued RKHSs and to injective width-expanding maps. Our Brownian result instead uses an anchored one-dimensional Cameron--Martin geometry and therefore requires a separate composition analysis.

\paragraph{Vector-valued RKHSs and multitask learning.}
Matrix-valued kernels and vector-valued RKHSs provide a standard framework for coupled outputs and multitask learning~\citep{Evgeniou2005,Micchelli2005,Caponnetto2008,argyriou2006multitask,Argyriou2008convex}. Existing guarantees include Rademacher bounds for linear multitask classes~\citep{Maurer2006Bounds}, excess-risk bounds for trace-norm regularization~\citep{Pontil2013Excess}, representation-learning guarantees~\citep{Maurer2016Benefit}, and local-complexity analyses~\citep{Yousefi2018Local}. These works characterize task sharing but do not analyze deep compositions through layerwise Koopman operators.

\paragraph{Operator learning and transfer.}
Representer theorems for learning Koopman operators show how Hilbert-space operator problems can reduce to finite rank~\citep{khosravi2023representer}. We use the same general projection principle in a vector-valued multitask setting. The transferable object is a bounded operator between Hilbert function spaces, rather than only an output kernel or a finite-dimensional latent vector.

\section{Preliminaries}\label{sec3}
This section fixes the notation, the vector-valued RKHS conventions, the Rademacher complexity, and the Koopman representation used in the main results.

\subsection{Notation and function spaces}\label{sec:notation}
The natural numbers are $\N\coloneqq\left\{0,1,2,\ldots\right\}$ and the positive integers are $\Zp\coloneqq\left\{1,2,\ldots\right\}$. For $n\in\Zp$, write $[n]\coloneqq\left\{1,\ldots,n\right\}$. The nonnegative reals are $\Rnn$. For a bounded linear operator $A$ between Hilbert spaces, $\ra\left(A\right)$, $\ker\left(A\right)$, and $\|A\|$ denote its range, kernel, and operator norm. If $\b W\in\R^{d\times d'}$ is injective with $d\ge d'$, define its range-volume determinant by
\begin{align*}
\left|\det\left(\b W\right)\right|
\coloneqq
\det\left(\b W^{\top}\b W\right)^{1/2}.
\end{align*}
The cone of symmetric positive semidefinite $m\times m$ matrices is $\mathbb S_+^m$; $\b M\succ0$ means that $\b M$ is strictly positive definite.

A scalar kernel $k:\mX\times\mX\to\R$ is symmetric and positive semidefinite. Its RKHS is denoted by $\mH_k$, and $k_{\b x}\coloneqq k\left(\cdot,\b x\right)$ satisfies
\begin{align*}
h\left(\b x\right)
=
\left\langle h,k_{\b x}\right\rangle_{\mH_k}.
\end{align*}
A matrix-valued kernel $K:\mX\times\mX\to\R^{m\times m}$ satisfies $K\left(\b x,\b x'\right)=K\left(\b x',\b x\right)^{\top}$ and
\begin{align*}
\sum_{i,j=1}^{n}
\b y_i^{\top}K\left(\b x_i,\b x_j\right)\b y_j
\ge 0
\end{align*}
for all $n\in\Zp$, $\b x_i\in\mX$, and $\b y_i\in\R^m$. The induced vvRKHS $\mH_K$ satisfies the reproducing identity
\begin{align*}
\left\langle f,K\left(\cdot,\b x\right)\b y\right\rangle_{\mH_K}
=
f\left(\b x\right)^{\top}\b y.
\end{align*}

For a separable kernel $K\left(\b x,\b x'\right)=k\left(\b x,\b x'\right)\b M$ with $\b M\succ0$, the vvRKHS consists of $f=\left(f_1,\ldots,f_m\right)$ with $f_j\in\mH_k$, and
\begin{align*}
\left\langle f,g\right\rangle_{\mH_K}
&=
\sum_{i,j=1}^{m}
\left(\b M^{-1}\right)_{ij}
\left\langle f_i,g_j\right\rangle_{\mH_k},
\\
\|f\|_{\mH_K}^{2}
&=
\sum_{i,j=1}^{m}
\left(\b M^{-1}\right)_{ij}
\left\langle f_i,f_j\right\rangle_{\mH_k}.
\end{align*}
The main Sobolev and Brownian theorems assume $\b M\succ0$, so the weighted norm is nondegenerate and no quotient-space convention is needed. The later target-transfer proposition permits $\b M\succeq0$ because its proof uses only the matrix-valued reproducing property and the diagonal trace bound.

For every ambient dimension $d$ and order $s>d/2$, let $k_s$ denote the Bessel-potential Sobolev kernel normalized so that its RKHS is $H^s\left(\R^d\right)$ with the exact Fourier norm displayed below. The ambient dimension is understood from context. We write
\begin{align*}
K_s\left(\b x,\b x'\right)
=
k_s\left(\b x,\b x'\right)\b M
\end{align*}
for the separable matrix-valued Sobolev kernel and denote its vvRKHS by $\HMs{s}\left(\R^d,\R^m\right)$. With the fixed Fourier-transform normalization used to define $k_s$, its norm is
\begin{align*}
\|h\|_{\HMs{s}\left(\R^d,\R^m\right)}^2
=
\int_{\R^d}
\left(1+\|\bomega\|_2^2\right)^s
\widehat h\left(\bomega\right)^*
\b M^{-1}
\widehat h\left(\bomega\right)
\,\d\bomega.
\end{align*}
For a linear subspace $\mS\subset\R^d$, $\HMs{s}\left(\mS,\R^m\right)$ is defined after identifying $\mS$ isometrically with $\R^{\dim\left(\mS\right)}$. A multi-index is $\alpha\in\N^d$, with $|\alpha|\coloneqq\sum_{j=1}^{d}\alpha_j$, and $D^\alpha$ denotes the corresponding weak derivative.

Let $\mH_{K_1}$ and $\mH_{K_2}$ be vvRKHSs on $\R^{d_1}$ and $\R^{d_2}$. For a measurable map $\phi:\R^{d_1}\to\R^{d_2}$, define
\begin{align*}
\mD_\phi
\coloneqq
\left\{g\in\mH_{K_2}:g\circ\phi\in\mH_{K_1}\right\}.
\end{align*}
The Koopman operator is $\mK_\phi:\mD_\phi\to\mH_{K_1}$, $\mK_\phi g\coloneqq g\circ\phi$. For $\b a\in\R^d$, the translation $\b b\left(\b x\right)=\b x+\b a$ induces $\mK_{\b b}g=g\circ\b b$.

The signed one-dimensional Brownian kernel is
\begin{align*}
\kB\left(x,x'\right)
\coloneqq
\frac{|x|+|x'|-|x-x'|}{2}.
\end{align*}
On $\mX=\left[-R,R\right]$, the separable kernel $K\left(x,x'\right)=\kB\left(x,x'\right)\b M$ induces the vector-valued Brownian RKHS $\HB\left(\mX,\R^m\right)$. This is the usual anchored Cameron--Martin geometry associated with Brownian covariance kernels~\citep{vanderVaart2008}; the exact two-sided signed-kernel characterization needed here is proved in \Cref{lem:BrownianCameronMartin}. For a scalar differentiable map $\sigma$, $\|\sigma'\|_\infty$ denotes the supremum of $|\sigma'|$ over its stated domain.

\subsection{Vector-valued Rademacher complexity}\label{sec:rad}
Let $\mD_n=\left\{\b x_1,\ldots,\b x_n\right\}\subset\mX$ be fixed. For each $i\in[n]$, let $\bsigma_i=\left(\sigma_{i1},\ldots,\sigma_{im}\right)$ have independent coordinates uniformly distributed on $\left\{-1,+1\right\}$, and assume that the vectors $\bsigma_1,\ldots,\bsigma_n$ are independent.

\begin{definition}[Empirical vector-valued Rademacher complexity]
For a class $\F$ of functions $f:\mX\to\R^m$, define
\begin{align*}
\eR\left(\F\right)
\coloneqq
\RE\left[
\sup_{f\in\F}
\frac{1}{n}
\left|
\sum_{i=1}^{n}
\left\langle\bsigma_i,f\left(\b x_i\right)\right\rangle_{\R^m}
\right|
\right].
\end{align*}
\end{definition}

The absolute value is convenient for the operator estimates below. All principal classes are symmetric in the terminal map, so it does not change their order of complexity.

\subsection{Koopman representation of deep networks}\label{sec:koopman-setup}
Fix positive integers $d_0,\ldots,d_L$ and $m$, and let $\mX\subset\R^{d_0}$ be the input domain. We consider functions of the form
\begin{align*}
f
=
g
\circ
\b b_L
\circ
\b W_L
\circ
\sigma_{L-1}
\circ
\cdots
\circ
\sigma_1
\circ
\b b_1
\circ
\b W_1,
\end{align*}
where $\b W_l\in\R^{d_l\times d_{l-1}}$, $\b b_l\left(\b x\right)=\b x+\b a_l$ with $\b a_l\in\R^{d_l}$, $\sigma_l:\R^{d_l}\to\R^{d_l}$, and $g:\R^{d_L}\to\R^m$. The composition order is from right to left. At the function-space level,
\begin{align}
\label{eq:koopman-factorization}
f
=
\mK_{\b W_1}
\mK_{\b b_1}
\mK_{\sigma_1}
\cdots
\mK_{\b W_{L-1}}
\mK_{\b b_{L-1}}
\mK_{\sigma_{L-1}}
\mK_{\b W_L}
\mK_{\b b_L}
g.
\end{align}

For the Sobolev results, fix $\b M\succ0$ and exponents $s_l>d_l/2$. The space at layer $l$ is $\HMs{s_l}\left(\R^{d_l},\R^m\right)$, induced by $K_{s_l}=k_{s_l}\b M$. We make the following standing assumption.

\begin{assumption}[Sobolev admissibility]\label{assum:sobolev-admissibility}
For $l\in[L-1]$, the activation Koopman operator
\begin{align*}
\mK_{\sigma_l}:
\HMs{s_l}\left(\R^{d_l},\R^m\right)
\longrightarrow
\HMs{s_l}\left(\R^{d_l},\R^m\right)
\end{align*}
is bounded. The terminal map belongs to $\HMs{s_L}\left(\R^{d_L},\R^m\right)$. Sufficient smoothness conditions for activation boundedness are given in \Cref{lem:bounded-koopman-activation}.
\end{assumption}

For $B_g>0$, let $\F\left(B_g\right)$ be the class of all functions represented by the displayed architecture with arbitrary admissible biases, weights satisfying the layer-specific conditions stated later, and terminal maps satisfying
\begin{align*}
\|g\|_{\HMs{s_L}\left(\R^{d_L},\R^m\right)}
\le B_g.
\end{align*}
We suppress the argument $B_g$ when it is fixed. Smooth compactly supported vector-valued functions provide nontrivial admissible terminal maps.

\begin{assumption}[Bounded input kernel diagonal]\label{assum:kernel-diagonal}
There exists $\kappa>0$ such that
\begin{align*}
k_{s_0}\left(\b x,\b x\right)
\le\kappa,
\qquad
\b x\in\mX.
\end{align*}
\end{assumption}

The Koopman factorization isolates the contribution of each layer: translations, activations, and linear maps are bounded separately, and their norms are combined by submultiplicativity. The trace of $\b M$ enters through the vector-valued RKHS evaluation term, whereas determinant and Sobolev-symbol factors enter through the linear composition operators.

\section{Vector-Valued Sobolev Koopman Bounds}\label{sec:koopman-sobolev}
The Sobolev estimates combine a vector-valued RKHS evaluation term with layerwise Koopman operator norms. The task matrix appears through $\Tr\left(\b M\right)$ and the $\b M^{-1}$-weighted terminal norm. The linear maps contribute Sobolev-symbol ratios and determinant terms, while the nonlinearities contribute their composition-operator norms.

\subsection{Invertible deep networks}\label{sub:1}
Assume $d_l=d$ for $l=0,\ldots,L$, where $d\in\Zp$, and impose the monotone Sobolev-order condition
\begin{align*}
0
\le
s_0
\le
s_1
\le
\cdots
\le
s_L,
\qquad
s_l>
\frac{d}{2},
\qquad
l=0,\ldots,L.
\end{align*}
For an invertible linear map, this monotonicity ensures that every Sobolev-symbol supremum appearing below is finite. For $C,D>0$, define
\begin{align*}
\W
\coloneqq
\left\{
\b W\in\R^{d\times d}
\;\middle|\;
\|\b W\|\le C,
\quad
|\det\left(\b W\right)|\ge D
\right\}.
\end{align*}
The corresponding class is
\begin{align*}
\F_{\scriptscriptstyle \mm{inv}}
\coloneqq
\left\{
f\in\F\left(B_g\right)
\;\middle|\;
\left(\b W_1,\ldots,\b W_L\right)\in\W^L
\right\}.
\end{align*}

\begin{theorem}[Invertible Sobolev bound]\label{inv}
Under \Cref{assum:sobolev-admissibility,assum:kernel-diagonal} and the monotone Sobolev-order condition above,
\begin{align*}
\eR\left(\F_{\scriptscriptstyle \mm{inv}}\right)
&\le
B_g
\sqrt{
\frac{
\kappa\Tr\left(\b M\right)
}{n}
}
\prod_{l=1}^{L-1}
\|\mK_{\sigma_l}\|
\\
&\quad\cdot
\sup_{\left(\b W_1,\ldots,\b W_L\right)\in\W^L}
\prod_{l=1}^{L}
\left[
\sup_{\bomega\in\R^d}
\left(
\frac{
\left(1+\|\b W_l^{\top}\bomega\|_2^2\right)^{s_{l-1}}
}{
\left(1+\|\bomega\|_2^2\right)^{s_l}
}
\right)^{1/2}
\frac{1}{|\det\left(\b W_l\right)|^{1/2}}
\right].
\end{align*}
Every translation Koopman operator is an isometry on the weighted Sobolev spaces, so no bias-dependent factor appears.
\end{theorem}

\begin{corollary}[Common Sobolev order]\label{cor1}
Suppose $s_l=s>d/2$ for $l=0,\ldots,L$. Then
\begin{align*}
\eR\left(\F_{\scriptscriptstyle \mm{inv}}\right)
\le
B_g
\sqrt{
\frac{
\kappa\Tr\left(\b M\right)
}{n}
}
\left(
\frac{\max\left\{1,C^s\right\}}{D^{1/2}}
\right)^L
\prod_{l=1}^{L-1}
\|\mK_{\sigma_l}\|.
\end{align*}
\end{corollary}

\subsection{Injective width-expanding networks}\label{sub:2}
Assume $d_0\le d_1\le\cdots\le d_L$ and set $q_l\coloneqq d_l-d_{l-1}$. In addition to $s_j>d_j/2$, assume
\begin{align*}
s_l>\frac{q_l}{2},
\qquad
0\le s_{l-1}\le s_l-\frac{q_l}{2},
\qquad
l\in[L].
\end{align*}
These conditions are sufficient for the standard Sobolev trace map from $\R^{d_l}$ to a $d_{l-1}$-dimensional linear subspace to be bounded at the stated orders~\citep{wendland2005scattered}.

For $C,D>0$, define
\begin{align*}
\Wl
\coloneqq
\left\{
\b W\in\R^{d_l\times d_{l-1}}
\;\middle|\;
\b W\text{ is injective},
\quad
\|\b W\|\le C,
\quad
\det\left(\b W^{\top}\b W\right)^{1/2}\ge D
\right\}.
\end{align*}
For $\b W\in\Wl$, let
\begin{align*}
R_{\b W}:
\HMs{s_l}\left(\R^{d_l},\R^m\right)
\longrightarrow
\HMs{s_{l-1}}\left(\ra\left(\b W\right),\R^m\right),
\qquad
R_{\b W}h
\coloneqq
h|_{\ra\left(\b W\right)}.
\end{align*}
Define the uniform restriction constant
\begin{align*}
\mm G_l
\coloneqq
\sup_{\b W\in\Wl}
\|R_{\b W}\|.
\end{align*}
The trace theorem and rotational invariance of the Sobolev norm imply $\mm G_l<\infty$.

The injective class is
\begin{align*}
\F_{\scriptscriptstyle \mm{inj}}
\coloneqq
\left\{
f\in\F\left(B_g\right)
\;\middle|\;
\left(\b W_1,\ldots,\b W_L\right)
\in
\prod_{l=1}^{L}\Wl
\right\}.
\end{align*}

\begin{theorem}[Injective Sobolev bound]\label{inj}
Under \Cref{assum:sobolev-admissibility,assum:kernel-diagonal} and the trace conditions above,
\begin{align*}
\eR\left(\F_{\scriptscriptstyle \mm{inj}}\right)
&\le
B_g
\sqrt{
\frac{
\kappa\Tr\left(\b M\right)
}{n}
}
\prod_{l=1}^{L-1}
\|\mK_{\sigma_l}\|
\\
&\quad\cdot
\sup_{\left(\b W_1,\ldots,\b W_L\right)
\in\prod_{l=1}^{L}\Wl}
\prod_{l=1}^{L}
\left[
\mm G_l
\sup_{\bomega\in\ra\left(\b W_l\right)}
\left(
\frac{
1+\|\b W_l^{\top}\bomega\|_2^2
}{
1+\|\bomega\|_2^2
}
\right)^{s_{l-1}/2}
\frac{1}{
\det\left(\b W_l^{\top}\b W_l\right)^{1/4}
}
\right].
\end{align*}
\end{theorem}

Fix $l\in[L]$, $\b W_l\in\Wl$, and $h\in\HMs{s_l}\left(\R^{d_l},\R^m\right)$. Define the realized restriction ratio
\begin{align*}
\gamma_l\left(h,\b W_l\right)
\coloneqq
\begin{cases}
0,
& h=0,
\\
\displaystyle
\frac{
\|R_{\b W_l}h\|_{\HMs{s_{l-1}}\left(\ra\left(\b W_l\right),\R^m\right)}
}{
\|h\|_{\HMs{s_l}\left(\R^{d_l},\R^m\right)}
},
& h\ne0.
\end{cases}
\end{align*}
The factorization $\mK_{\b W_l}=P_{\b W_l}R_{\b W_l}$ proved in \Cref{lem:complete-injective-layer} then gives the function-dependent estimate
\begin{align*}
\|\mK_{\b W_l}h\|_{\HMs{s_{l-1}}\left(\R^{d_{l-1}},\R^m\right)}
&\le
\gamma_l\left(h,\b W_l\right)
\sup_{\bomega\in\ra\left(\b W_l\right)}
\left(
\frac{
1+\|\b W_l^{\top}\bomega\|_2^2
}{
1+\|\bomega\|_2^2
}
\right)^{s_{l-1}/2}
\frac{
\|h\|_{\HMs{s_l}\left(\R^{d_l},\R^m\right)}
}{
\det\left(\b W_l^{\top}\b W_l\right)^{1/4}
}.
\end{align*}
Moreover,
\begin{align*}
0
\le
\gamma_l\left(h,\b W_l\right)
\le
\|R_{\b W_l}\|
\le
\mm G_l.
\end{align*}
In a fixed network, $h$ is the actual downstream function presented to $\mK_{\b W_l}$ in the Koopman chain. The ratio therefore records the restriction cost of that particular realization. It is function-dependent and cannot replace the uniform constant $\mm G_l$ in the class-level Rademacher theorem.

\begin{remark}[Rank-deficient weight matrices]
Injectivity is used to obtain a bounded pullback between the stated Sobolev spaces and a nonzero range-volume determinant. For a rank-deficient matrix, the factor $\det\left(\b W^{\top}\b W\right)^{-1/4}$ is singular, and the theorem above does not apply. The stabilized expression $\det\left(\b I+\b W^{\top}\b W\right)^{-1/4}$ can be used as a finite empirical surrogate, as in the experiments, but this substitution alone does not prove boundedness of the rank-deficient Koopman operator. A rigorous non-injective extension requires a different function-space or quotient-space analysis and is outside the scope of this paper.
\end{remark}

\section{Vector-Valued Brownian/Cameron--Martin Koopman Bound}\label{sec:koopman-brownian}
The Brownian result concerns a different hypothesis space from the Sobolev results. We work on the interval $\mX=\left[-R,R\right]$ with $R>0$, the signed Brownian kernel
\begin{align*}
\kB\left(x,x'\right)
=
\frac{|x|+|x'|-|x-x'|}{2},
\end{align*}
and the matrix-valued kernel $K\left(x,x'\right)=\kB\left(x,x'\right)\b M$, where $\b M\succ0$. By \Cref{lem:BrownianCameronMartin},
\begin{align*}
\HB\left(\mX,\R^m\right)
=
\left\{
f:\mX\to\R^m
\;\middle|\;
f\left(0\right)=0,
\ f\text{ is absolutely continuous},
\ f'\in L^2\left(\mX,\R^m\right)
\right\},
\end{align*}
with the exact norm
\begin{align*}
\|f\|_{\HB\left(\mX,\R^m\right)}^2
=
\int_{-R}^{R}
f'\left(t\right)^{\top}\b M^{-1}f'\left(t\right)
\,\d t.
\end{align*}

For constants $0<D\le C\le1$, define
\begin{align*}
\W_{\mathrm B}
\coloneqq
\left\{W\in\R:D\le|W|\le C\right\}.
\end{align*}
The restriction $C\le1$ ensures that $x\mapsto Wx$ maps $\mX$ into itself. For $l\in[L-1]$, assume that $\sigma_l:\mX\to\mX$ is a $C^1$ diffeomorphism satisfying $\sigma_l\left(0\right)=0$. The anchoring condition is necessary because every function in the Brownian RKHS vanishes at the origin. Bias translations are therefore not included in this regime.

Define
\begin{align*}
\F_{\mathrm{inv}}^{\scriptscriptstyle (\mathrm B)}
\coloneqq
\left\{
\mK_{W_1}
\mK_{\sigma_1}
\cdots
\mK_{\sigma_{L-1}}
\mK_{W_L}g
\;\middle|\;
\|g\|_{\HB\left(\mX,\R^m\right)}\le B_g,
\quad
\left(W_1,\ldots,W_L\right)\in\W_{\mathrm B}^L
\right\}.
\end{align*}

\begin{theorem}[Brownian/Cameron--Martin bound]\label{thm:BrownianKoopman}
For every fixed sample $\left\{x_1,\ldots,x_n\right\}\subset\mX$,
\begin{align*}
\eR\left(\F_{\mathrm{inv}}^{\scriptscriptstyle (\mathrm B)}\right)
&\le
B_g
\sqrt{
\frac{
R\Tr\left(\b M\right)
}{n}
}
\prod_{l=1}^{L-1}
\|\sigma_l'\|_\infty^{1/2}
\\
&\quad\cdot
\sup_{\left(W_1,\ldots,W_L\right)\in\W_{\mathrm B}^L}
\prod_{l=1}^{L}|W_l|^{1/2}.
\end{align*}
\end{theorem}

The layerwise estimates behind the theorem are
\begin{align*}
\|\mK_W\|
&\le |W|^{1/2},
\\
\|\mK_\sigma\|
&\le\|\sigma'\|_\infty^{1/2}.
\end{align*}
They follow directly from the anchored derivative norm; no compactly supported extension, Fourier seminorm, or inverse-derivative factor is required.

\begin{remark}[Scope of the Brownian comparison]
The Sobolev and Brownian results act on different RKHSs, use different domain conditions, and cover different architecture classes. The Brownian theorem is one-dimensional, excludes bias translations, and requires every layer map to preserve $\left[-R,R\right]$ and the anchor at zero. It should therefore be interpreted as a separate operator-theoretic regime, not as a uniform improvement over the multidimensional Sobolev bounds.
\end{remark}

\section{Shared Operator Learning and Transfer}\label{sec:shared-transfer}
We next consider a shared operator learned from several source tasks. The purpose is to record the Hilbert-space representer structure and the complexity of the target class induced by that operator. The section does not claim that operator sharing improves every transfer problem.

\subsection{Shared operator-learning problem}
Let $\mH_{\mathrm{out}}$ be a Hilbert space of terminal maps. Examples include $\HMs{s_L}\left(\R^{d_L},\R^m\right)$ and $\HB\left(\mX,\R^m\right)$. Let $\mH_K$ be a vvRKHS of functions from $\mX$ to $\R^m$, induced by a matrix-valued kernel $K$. We write
\begin{align*}
\m L_2\left(\mH_{\mathrm{out}},\mH_K\right)
\end{align*}
for the Hilbert space of Hilbert--Schmidt operators from $\mH_{\mathrm{out}}$ to $\mH_K$. If $\left(e_j\right)_{j\in J}$ is any orthonormal basis of $\mH_{\mathrm{out}}$, its inner product and norm are
\begin{align*}
\left\langle S,T\right\rangle_{\mathrm{HS}}
&\coloneqq
\sum_{j\in J}
\left\langle Se_j,Te_j\right\rangle_{\mH_K},
\\
\|T\|_{\mathrm{HS}}^2
&\coloneqq
\sum_{j\in J}
\|Te_j\|_{\mH_K}^2,
\end{align*}
and these definitions are independent of the chosen orthonormal basis. We learn an operator
\begin{align*}
T\in\m L_2\left(\mH_{\mathrm{out}},\mH_K\right).
\end{align*}
For $u\in\mH_K$ and $v\in\mH_{\mathrm{out}}$, the rank-one operator $u\otimes v$ is defined by
\begin{align*}
\left(u\otimes v\right)h
=
\left\langle v,h\right\rangle_{\mH_{\mathrm{out}}}u,
\qquad
h\in\mH_{\mathrm{out}}.
\end{align*}

Let $N_{\mathrm s}\in\Zp$ be the number of source tasks. For task $t\in\left[N_{\mathrm s}\right]$, let $g_t\in\mH_{\mathrm{out}}$ be its terminal representation and let $\left(\b x_{ti},\b y_{ti}\right)\in\mX\times\R^m$, $i\in[n_t]$, be its training data. For $\lambda>0$, consider
\begin{align}
\label{eq:shared-transfer-opt}
\underset{T\in\m L_2\left(\mH_{\mathrm{out}},\mH_K\right)}{\operatorname{minimize}}
\sum_{t=1}^{N_{\mathrm s}}
\sum_{i=1}^{n_t}
\ell\left(\b y_{ti},\left(Tg_t\right)\left(\b x_{ti}\right)\right)
+
\lambda\|T\|_{\mathrm{HS}}^2,
\end{align}
where $\ell:\R^m\times\R^m\to\left[0,\infty\right)$ is finite-valued, continuous, and convex in its second argument.

\begin{proposition}[Well-posedness of shared operator learning]\label{prop:shared-transfer-well-posed}
Problem \eqref{eq:shared-transfer-opt} has a unique minimizer in $\m L_2\left(\mH_{\mathrm{out}},\mH_K\right)$.
\end{proposition}

\subsection{Representer theorem}
\begin{theorem}[Shared-operator representer theorem]\label{thm:shared-transfer-representer}
The unique minimizer $\widehat T$ of \eqref{eq:shared-transfer-opt} has a finite-rank representation
\begin{align*}
\widehat T
=
\sum_{t=1}^{N_{\mathrm s}}
\sum_{i=1}^{n_t}
\sum_{a=1}^{m}
c_{tia}
\left(
K\left(\cdot,\b x_{ti}\right)\bea
\right)
\otimes g_t,
\end{align*}
where $\left(\bea\right)_{a=1}^{m}$ is the standard basis of $\R^m$. Consequently, for every $g\in\mH_{\mathrm{out}}$ and $\b x\in\mX$,
\begin{align*}
\left(\widehat T g\right)\left(\b x\right)
=
\sum_{t=1}^{N_{\mathrm s}}
\sum_{i=1}^{n_t}
\sum_{a=1}^{m}
c_{tia}
\left\langle g_t,g\right\rangle_{\mH_{\mathrm{out}}}
K\left(\b x,\b x_{ti}\right)\bea.
\end{align*}
\end{theorem}

\subsection{Finite-dimensional reduction}
\begin{theorem}[Squared-loss coefficient problem]\label{thm:shared-transfer-finite}
Assume $\ell\left(\b y,\b z\right)=\|\b y-\b z\|_2^2$. The coefficients of any representer expansion of the unique minimizer in \Cref{thm:shared-transfer-representer} solve
\begin{align*}
\min_c\;&
\sum_{t=1}^{N_{\mathrm s}}
\sum_{i=1}^{n_t}
\left\|
\b y_{ti}
-
\sum_{t'=1}^{N_{\mathrm s}}
\sum_{j=1}^{n_{t'}}
\sum_{a=1}^{m}
c_{t'ja}
\left\langle g_{t'},g_t\right\rangle_{\mH_{\mathrm{out}}}
K\left(\b x_{ti},\b x_{t'j}\right)\bea
\right\|_2^2
\\
&+
\lambda
\sum_{t=1}^{N_{\mathrm s}}
\sum_{i=1}^{n_t}
\sum_{a=1}^{m}
\sum_{t'=1}^{N_{\mathrm s}}
\sum_{j=1}^{n_{t'}}
\sum_{b=1}^{m}
c_{tia}c_{t'jb}
\left\langle g_t,g_{t'}\right\rangle_{\mH_{\mathrm{out}}}
\bea^{\top}
K\left(\b x_{ti},\b x_{t'j}\right)
\beb.
\end{align*}
Conversely, every coefficient tensor $c$ defines a finite-rank operator of the representer form, so the displayed problem has the same optimal value as \eqref{eq:shared-transfer-opt} under squared loss.
\end{theorem}

\subsection{Target-transfer bound}
Let $K\left(\b x,\b x'\right)=k\left(\b x,\b x'\right)\b M$ with $\b M\succeq0$ and $k\left(\b x,\b x\right)\le\kappa$. Assume that $\widehat T$ is measurable with respect to the source data and is independent of the target sample. Conditionally on $\widehat T$, define
\begin{align*}
\F_{\widehat T}\left(B\right)
\coloneqq
\left\{
\widehat T g:
\|g\|_{\mH_{\mathrm{out}}}\le B
\right\}.
\end{align*}
Let
\begin{align*}
R_0\left(f\right)
\coloneqq
\mathbb E\left[\ell\left(Y,f\left(X\right)\right)\right],
\qquad
R_0^\star
\coloneqq
\inf_f R_0\left(f\right),
\end{align*}
where the infimum is over all measurable predictors for which the risk is defined. Let $\widehat f_0$ be a measurable empirical risk minimizer over $\F_{\widehat T}\left(B\right)$ on an independent target sample of size $n_0$. Define the approximation error
\begin{align*}
A_{\widehat T}\left(B\right)
\coloneqq
\inf_{\|g\|_{\mH_{\mathrm{out}}}\le B}
R_0\left(\widehat T g\right)
-
R_0^\star.
\end{align*}

\begin{proposition}[Conditional target bound]\label{prop:shared-transfer-bound}
Assume $0\le\ell\le C_\ell$ and that $\ell\left(y,\cdot\right)$ is $L_\ell$-Lipschitz with respect to the Euclidean norm. Assume also that all suprema appearing in the Rademacher complexities and uniform-deviation quantities used below are measurable. Conditionally on $\widehat T$, with probability at least $1-\delta$ over the target sample,
\begin{align*}
R_0\left(\widehat f_0\right)-R_0^\star
&\le
A_{\widehat T}\left(B\right)
+
4\sqrt{2}\,L_\ell B
\|\widehat T\|
\sqrt{
\frac{
\kappa\Tr\left(\b M\right)
}{n_0}
}
\\
&\quad+
2C_\ell
\sqrt{
\frac{
\log\left(2/\delta\right)
}{2n_0}
}.
\end{align*}
The same statement holds with $\|\widehat T\|_{\mathrm{HS}}$ in place of $\|\widehat T\|$.
\end{proposition}

\begin{remark}
The proposition separates approximation from estimation. It does not assert that a shared operator always improves target performance. It states that, when a useful source-learned operator has controlled norm, the target class induced by its image has controlled complexity. If $\widehat T$ is itself a Koopman chain satisfying \Cref{inv,inj,thm:BrownianKoopman}, its operator norm can be bounded by the corresponding layerwise estimates.
\end{remark}

\section{Empirical Proxy Study}\label{sec:experiments}
The experiments examine stabilized numerical factors motivated by the layer geometry in the theory. They are not direct evaluations of \Cref{inv,inj,thm:BrownianKoopman}: the experimental architectures include rectangular and rank-deficient layers, and the determinant is stabilized by adding the identity.

\begin{figure}[tbp]
\centering
\begin{minipage}{0.47\linewidth}
\centering
\includegraphics[width=0.90\linewidth]{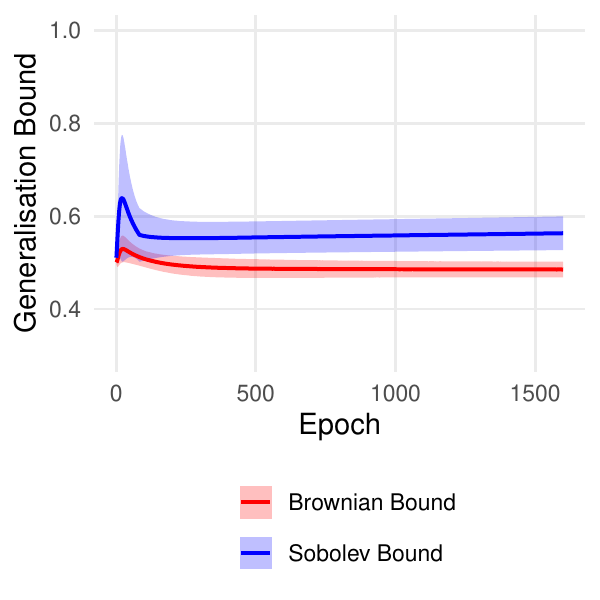}
\par\smallskip
(a)
\end{minipage}
\hfill
\begin{minipage}{0.47\linewidth}
\centering
\includegraphics[width=0.84\linewidth]{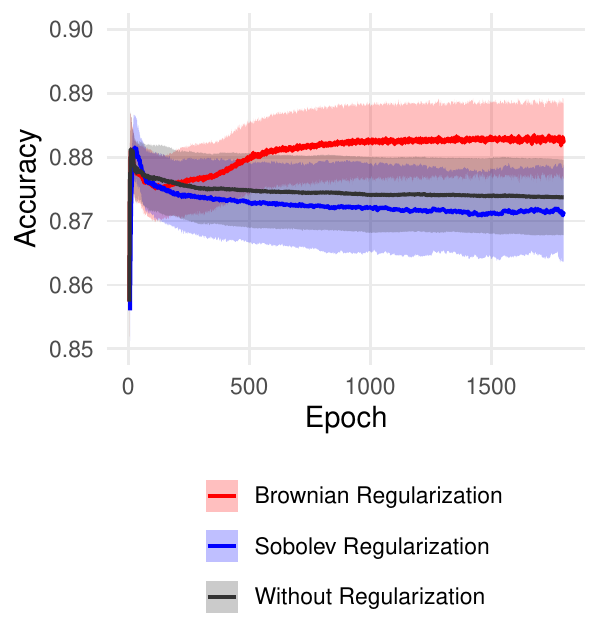}
\par\smallskip
(b)
\end{minipage}
\caption{Empirical behavior of stabilized complexity proxies. (a) Sobolev-inspired and Brownian-inspired proxy values during training on the synthetic regression problem. (b) MNIST test performance for the unregularized baseline and for models trained with the two proxy penalties. The labels inside the original plots use the shorthand ``bound,'' but the quantities are empirical surrogates rather than evaluations of the proved theorems.}
\label{fig:experiment}
\end{figure}

\subsection{Synthetic regression}
Following the experimental pattern of \citet{hashimoto2024koopmanbased}, we consider the target
\begin{align*}
t\left(\b x\right)
\coloneqq
\exp\left(-\|2\b x-\b 1\|_2^2\right),
\qquad
\b x\in\R^3.
\end{align*}
The network is
\begin{align*}
f\left(\b x\right)
=
g\left(\b W_2\sigma\left(\b W_1\b x+\b b_1\right)+\b b_2\right),
\end{align*}
with $\b W_1\in\R^{3\times3}$, $\b W_2\in\R^{6\times3}$, $\b b_1\in\R^3$, $\b b_2\in\R^6$, and $g\left(\b z\right)=\exp\left(-\|\b z\|_2^2\right)$. The weights use orthogonal initialization~\citep{saxe2014exact}, the biases use a uniform initialization, and the activation is the smooth Leaky ReLU of \citet{biswas2022smooth}. Training runs for $1600$ epochs with learning rate $3\times10^{-3}$ and an $L_2$ penalty of $10^{-4}$.

For descriptive comparison, define
\begin{align*}
\mathrm{SP}
&\coloneqq
\prod_{l=1}^{L}
\frac{
\|\b W_l\|^{s_l}
}{
\det\left(\b I+\b W_l^{\top}\b W_l\right)^{1/4}
},
\qquad
s_l\coloneqq\frac{d_l+0.1}{2},
\\
\mathrm{BP}
&\coloneqq
\prod_{l=1}^{L}
\frac{
\|\b W_l\|
}{
\det\left(\b I+\b W_l^{\top}\b W_l\right)^{1/4}
}.
\end{align*}
The identity stabilization keeps both proxies finite for rectangular or rank-deficient matrices. It is precisely why these expressions must not be read as the determinant factors in the proved injective theorem. Figure~\ref{fig:experiment}(a) reports the trajectories over five initializations. In this experiment, $\mathrm{BP}$ is numerically smaller and less variable than $\mathrm{SP}$. This is a descriptive observation about the chosen surrogates, not evidence that the Brownian theorem uniformly dominates the Sobolev theorem.

The MNIST regularization experiment and its architecture are described in \Cref{app:numerics}.

\FloatBarrier
\section{Limitations}\label{sec:limitations}
The Sobolev theorems require invertible or injective linear maps and bounded activation composition operators on the selected weighted Sobolev spaces. The injective theorem additionally requires the stated trace conditions, and its restriction constants may be large. Rank-deficient layers are not covered by replacing a singular determinant with a stabilized one; the experiments use that replacement only as a proxy.

The Brownian theorem is one-dimensional. It requires $|W_l|\le1$, activations that map $\left[-R,R\right]$ diffeomorphically onto itself, and the anchor condition $\sigma_l\left(0\right)=0$. Bias translations are excluded because they do not preserve the anchored RKHS. The theorem therefore does not directly cover standard unconstrained deep architectures.

All complexity bounds also depend on a prescribed terminal-map norm. The shared-transfer result is conditional on an operator learned independently of the target sample and separates estimation from approximation; it does not guarantee that the source-learned operator is useful for a particular target task.

\section{Conclusion}\label{sec5}
We developed operator-theoretic Rademacher complexity bounds for multi-output compositions in vector-valued RKHSs. In the Sobolev regime, the analysis separates task coupling, terminal-map size, activation composition norms, determinant-based volume distortion, and the restriction cost of injective width expansion. In the one-dimensional Brownian/Cameron--Martin regime, the exact anchored derivative norm gives direct linear and activation estimates without Sobolev exponents or Fourier extensions. We also established a finite-rank shared-operator representer theorem, its squared-loss coefficient reduction, and a source-conditioned target bound. The two function-space regimes have different assumptions and should be compared structurally rather than treated as uniformly ordered numerical bounds.

\bibliographystyle{plainnat}
\bibliography{references}

@article{Argyriou2008convex,
  author  = {Argyriou, Andreas and Evgeniou, Theodoros and Pontil, Massimiliano},
  title   = {Convex Multi-Task Feature Learning},
  journal = {Machine Learning},
  volume  = {73},
  number  = {3},
  pages   = {243--272},
  year    = {2008},
  doi     = {10.1007/s10994-007-5040-8}
}

@article{Caponnetto2008,
  author  = {Caponnetto, Andrea and Micchelli, Charles A. and Pontil, Massimiliano and Ying, Yiming},
  title   = {Universal Multi-Task Kernels},
  journal = {Journal of Machine Learning Research},
  volume  = {9},
  number  = {52},
  pages   = {1615--1646},
  year    = {2008}
}

@article{Evgeniou2005,
  author  = {Evgeniou, Theodoros and Micchelli, Charles A. and Pontil, Massimiliano},
  title   = {Learning Multiple Tasks with Kernel Methods},
  journal = {Journal of Machine Learning Research},
  volume  = {6},
  number  = {21},
  pages   = {615--637},
  year    = {2005}
}

@article{Fanghui2024,
  author  = {Liu, Fanghui and Dadi, Leello and Cevher, Volkan},
  title   = {Learning with Norm Constrained, Over-parameterized, Two-layer Neural Networks},
  journal = {Journal of Machine Learning Research},
  volume  = {25},
  number  = {138},
  pages   = {1--42},
  year    = {2024}
}

@inproceedings{Ju2022Robust,
  author    = {Ju, Haotian and Li, Dongyue and Zhang, Hongyang R.},
  title     = {Robust Fine-Tuning of Deep Neural Networks with Hessian-based Generalization Guarantees},
  booktitle = {Proceedings of the 39th International Conference on Machine Learning},
  series    = {Proceedings of Machine Learning Research},
  volume    = {162},
  pages     = {10431--10461},
  publisher = {PMLR},
  year      = {2022}
}

@article{Maurer2006Bounds,
  author  = {Maurer, Andreas},
  title   = {Bounds for Linear Multi-Task Learning},
  journal = {Journal of Machine Learning Research},
  volume  = {7},
  number  = {5},
  pages   = {117--139},
  year    = {2006}
}

@article{Maurer2016Benefit,
  author  = {Maurer, Andreas and Pontil, Massimiliano and Romera-Paredes, Bernardino},
  title   = {The Benefit of Multitask Representation Learning},
  journal = {Journal of Machine Learning Research},
  volume  = {17},
  number  = {81},
  pages   = {1--32},
  year    = {2016}
}

@article{Micchelli2005,
  author  = {Micchelli, Charles A. and Pontil, Massimiliano},
  title   = {On Learning Vector-Valued Functions},
  journal = {Neural Computation},
  volume  = {17},
  number  = {1},
  pages   = {177--204},
  year    = {2005},
  doi     = {10.1162/0899766052530802}
}

@inproceedings{Neyshabur2015,
  author    = {Neyshabur, Behnam and Tomioka, Ryota and Srebro, Nathan},
  title     = {Norm-Based Capacity Control in Neural Networks},
  booktitle = {Proceedings of the 28th Conference on Learning Theory},
  series    = {Proceedings of Machine Learning Research},
  volume    = {40},
  pages     = {1376--1401},
  publisher = {PMLR},
  year      = {2015}
}

@inproceedings{Pontil2013Excess,
  author    = {Pontil, Massimiliano and Maurer, Andreas},
  title     = {Excess Risk Bounds for Multitask Learning with Trace Norm Regularization},
  booktitle = {Proceedings of the 26th Annual Conference on Learning Theory},
  series    = {Proceedings of Machine Learning Research},
  volume    = {30},
  pages     = {55--76},
  publisher = {PMLR},
  year      = {2013}
}

@article{Yousefi2018Local,
  author  = {Yousefi, Niloofar and Lei, Yunwen and Kloft, Marius and Mollaghasemi, Mansooreh and Anagnostopoulos, Georgios C.},
  title   = {Local Rademacher Complexity-Based Learning Guarantees for Multi-Task Learning},
  journal = {Journal of Machine Learning Research},
  volume  = {19},
  number  = {38},
  pages   = {1--47},
  year    = {2018}
}

@inproceedings{argyriou2006multitask,
  author    = {Argyriou, Andreas and Evgeniou, Theodoros and Pontil, Massimiliano},
  title     = {Multi-Task Feature Learning},
  booktitle = {Advances in Neural Information Processing Systems},
  volume    = {19},
  pages     = {41--48},
  publisher = {MIT Press},
  year      = {2007}
}

@inproceedings{arora2018compression,
  author    = {Arora, Sanjeev and Ge, Rong and Neyshabur, Behnam and Zhang, Yi},
  title     = {Stronger Generalization Bounds for Deep Nets via a Compression Approach},
  booktitle = {Proceedings of the 35th International Conference on Machine Learning},
  series    = {Proceedings of Machine Learning Research},
  volume    = {80},
  pages     = {254--263},
  publisher = {PMLR},
  year      = {2018}
}

@article{bartlett2002rademacher,
  author  = {Bartlett, Peter L. and Mendelson, Shahar},
  title   = {Rademacher and Gaussian Complexities: Risk Bounds and Structural Results},
  journal = {Journal of Machine Learning Research},
  volume  = {3},
  pages   = {463--482},
  year    = {2002}
}

@inproceedings{bartlett2017spectrally,
  author    = {Bartlett, Peter L. and Foster, Dylan J. and Telgarsky, Matus J.},
  title     = {Spectrally-Normalized Margin Bounds for Neural Networks},
  booktitle = {Advances in Neural Information Processing Systems},
  volume    = {30},
  pages     = {6240--6249},
  year      = {2017}
}

@inproceedings{biswas2022smooth,
  author    = {Biswas, Koushik and Kumar, Sandeep and Banerjee, Shilpak and Pandey, Ashish Kumar},
  title     = {Smooth Maximum Unit: Smooth Activation Function for Deep Networks Using Smoothing Maximum Technique},
  booktitle = {Proceedings of the IEEE/CVF Conference on Computer Vision and Pattern Recognition},
  pages     = {784--793},
  year      = {2022},
  doi       = {10.1109/CVPR52688.2022.00087}
}

@article{golowich2018size,
  author  = {Golowich, Noah and Rakhlin, Alexander and Shamir, Ohad},
  title   = {Size-Independent Sample Complexity of Neural Networks},
  journal = {Information and Inference: A Journal of the IMA},
  volume  = {9},
  number  = {2},
  pages   = {473--504},
  year    = {2020}
}

@inproceedings{hashimoto2024koopmanbased,
  author    = {Hashimoto, Yuka and Sonoda, Sho and Ishikawa, Isao and Nitanda, Atsushi and Suzuki, Taiji},
  title     = {{Koopman}-Based Generalization Bound: New Aspect for Full-Rank Weights},
  booktitle = {International Conference on Learning Representations},
  year      = {2024}
}

@article{khosravi2023representer,
  author  = {Khosravi, Mohammad},
  title   = {Representer Theorem for Learning {Koopman} Operators},
  journal = {IEEE Transactions on Automatic Control},
  volume  = {68},
  number  = {5},
  pages   = {2995--3010},
  year    = {2023}
}

@inproceedings{kingma2017adammethodstochasticoptimization,
  author    = {Kingma, Diederik P. and Ba, Jimmy},
  title     = {Adam: A Method for Stochastic Optimization},
  booktitle = {International Conference on Learning Representations},
  year      = {2015}
}

@article{li2021orthogonal,
  author  = {Li, Shuai and Jia, Kui and Wen, Yuxin and Liu, Tongliang and Tao, Dacheng},
  title   = {Orthogonal Deep Neural Networks},
  journal = {IEEE Transactions on Pattern Analysis and Machine Intelligence},
  volume  = {43},
  number  = {4},
  pages   = {1352--1368},
  year    = {2021}
}

@inproceedings{maurer2016vectorcontraction,
  author    = {Maurer, Andreas},
  title     = {A Vector-Contraction Inequality for Rademacher Complexities},
  booktitle = {Algorithmic Learning Theory},
  series    = {Lecture Notes in Computer Science},
  volume    = {9925},
  pages     = {3--17},
  publisher = {Springer},
  year      = {2016},
  doi       = {10.1007/978-3-319-46379-7_1}
}

@book{mohri2018foundations,
  author    = {Mohri, Mehryar and Rostamizadeh, Afshin and Talwalkar, Ameet},
  title     = {Foundations of Machine Learning},
  edition   = {2},
  publisher = {MIT Press},
  address   = {Cambridge, MA},
  year      = {2018}
}

@inproceedings{saxe2014exact,
  author    = {Saxe, Andrew M. and McClelland, James L. and Ganguli, Surya},
  title     = {Exact Solutions to the Nonlinear Dynamics of Learning in Deep Linear Neural Networks},
  booktitle = {International Conference on Learning Representations},
  year      = {2014}
}

@incollection{vanderVaart2008,
  author    = {van der Vaart, Aad W. and van Zanten, J. Harry},
  title     = {Reproducing Kernel Hilbert Spaces of Gaussian Priors},
  booktitle = {Pushing the Limits of Contemporary Statistics: Contributions in Honor of Jayanta K. Ghosh},
  series    = {Institute of Mathematical Statistics Collections},
  volume    = {3},
  pages     = {200--222},
  publisher = {Institute of Mathematical Statistics},
  year      = {2008},
  doi       = {10.1214/074921708000000156}
}

@inproceedings{wei2019data,
  author    = {Wei, Colin and Ma, Tengyu},
  title     = {Data-Dependent Sample Complexity of Deep Neural Networks via Lipschitz Augmentation},
  booktitle = {Advances in Neural Information Processing Systems},
  volume    = {32},
  pages     = {9725--9736},
  year      = {2019}
}

@inproceedings{wei2020improved,
  author    = {Wei, Colin and Ma, Tengyu},
  title     = {Improved Sample Complexities for Deep Networks and Robust Classification via an All-Layer Margin},
  booktitle = {International Conference on Learning Representations},
  year      = {2020}
}

@book{wendland2005scattered,
  author    = {Wendland, Holger},
  title     = {Scattered Data Approximation},
  series    = {Cambridge Monographs on Applied and Computational Mathematics},
  volume    = {17},
  publisher = {Cambridge University Press},
  year      = {2004}
}

\appendix

\section{Proofs of the Sobolev Bounds}\label{app:sobolev-proofs}

This appendix proves the vector-valued Sobolev results. Every calculation is carried out in the $\b M^{-1}$-weighted Sobolev norm used in the theorem statements. The proofs first establish a common vvRKHS Rademacher estimate and then derive the layerwise composition bounds.

\subsection{A common vvRKHS Rademacher estimate}

\begin{lemma}[Operator-image Rademacher bound]\label{lem:operator-image-rademacher}
Let $\mH_G$ be a Hilbert space and let $\mH_K$ be a vvRKHS of functions from $\mX$ to $\R^m$ with matrix-valued kernel $K$. Let $\mT$ be a family of bounded linear operators from $\mH_G$ to $\mH_K$, and assume that
\begin{align*}
\sup_{T\in\mT}\|T\|
\le A.
\end{align*}
For $B>0$, define
\begin{align*}
\mF_{\mT}\left(B\right)
\coloneqq
\left\{
Tg:
T\in\mT,
\ \|g\|_{\mH_G}\le B
\right\}.
\end{align*}
Then, for every fixed sample $\left\{\b x_1,\ldots,\b x_n\right\}\subset\mX$,
\begin{align*}
\widehat{\mf R}_n^m\left(\mF_{\mT}\left(B\right)\right)
\le
\frac{BA}{n}
\RE\left[
\left\|
\sum_{i=1}^{n}
K\left(\cdot,\b x_i\right)\bsigma_i
\right\|_{\mH_K}
\right].
\end{align*}
If $K\left(\b x,\b x'\right)=k\left(\b x,\b x'\right)\b M$ with $\b M\succeq0$ and $k\left(\b x_i,\b x_i\right)\le\kappa$, then
\begin{align*}
\widehat{\mf R}_n^m\left(\mF_{\mT}\left(B\right)\right)
\le
BA
\sqrt{
\frac{
\kappa\Tr\left(\b M\right)
}{n}
}.
\end{align*}
\end{lemma}

\begin{proof}
Fix the sample and a realization of the Rademacher vectors. Define
\begin{align*}
V_{\bsigma}
\coloneqq
\sum_{i=1}^{n}
K\left(\cdot,\b x_i\right)\bsigma_i
\in\mH_K.
\end{align*}
For $T\in\mT$ and $g\in\mH_G$, the reproducing identity gives
\begin{align*}
\sum_{i=1}^{n}
\left\langle
\bsigma_i,
\left(Tg\right)\left(\b x_i\right)
\right\rangle_{\R^m}
&=
\sum_{i=1}^{n}
\left\langle
Tg,
K\left(\cdot,\b x_i\right)\bsigma_i
\right\rangle_{\mH_K}
\\
&=
\left\langle
Tg,
V_{\bsigma}
\right\rangle_{\mH_K}.
\end{align*}
The adjoint relation then yields
\begin{align*}
\left\langle
Tg,
V_{\bsigma}
\right\rangle_{\mH_K}
=
\left\langle
g,
T^*V_{\bsigma}
\right\rangle_{\mH_G}.
\end{align*}
Applying the Cauchy--Schwarz inequality in $\mH_G$ gives
\begin{align*}
\left|
\left\langle
g,
T^*V_{\bsigma}
\right\rangle_{\mH_G}
\right|
&\le
\|g\|_{\mH_G}
\|T^*V_{\bsigma}\|_{\mH_G}
\\
&\le
B\|T^*\|
\|V_{\bsigma}\|_{\mH_K}
\\
&=
B\|T\|
\|V_{\bsigma}\|_{\mH_K}
\\
&\le
BA
\|V_{\bsigma}\|_{\mH_K}.
\end{align*}
Taking the supremum over $T$ and $g$, dividing by $n$, and then taking the Rademacher expectation proves the first assertion.

Assume now that $K=k\b M$ and $k\left(\b x_i,\b x_i\right)\le\kappa$. Jensen's inequality for the concave square-root function gives
\begin{align*}
\RE\left[
\|V_{\bsigma}\|_{\mH_K}
\right]
\le
\left(
\RE\left[
\|V_{\bsigma}\|_{\mH_K}^2
\right]
\right)^{1/2}.
\end{align*}
Expanding the squared norm by the reproducing property gives
\begin{align*}
\|V_{\bsigma}\|_{\mH_K}^2
&=
\sum_{i=1}^{n}
\sum_{j=1}^{n}
\left\langle
K\left(\cdot,\b x_i\right)\bsigma_i,
K\left(\cdot,\b x_j\right)\bsigma_j
\right\rangle_{\mH_K}
\\
&=
\sum_{i=1}^{n}
\sum_{j=1}^{n}
\bsigma_i^{\top}
K\left(\b x_i,\b x_j\right)
\bsigma_j.
\end{align*}
When $i\ne j$, independence and centering imply
\begin{align*}
\RE\left[
\bsigma_i^{\top}
K\left(\b x_i,\b x_j\right)
\bsigma_j
\right]
=0.
\end{align*}
For a diagonal term, independence of the coordinates and the identity $\RE\left[\bsigma_i\bsigma_i^{\top}\right]=\b I_m$ give
\begin{align*}
\RE\left[
\bsigma_i^{\top}
K\left(\b x_i,\b x_i\right)
\bsigma_i
\right]
&=
\Tr\left(
K\left(\b x_i,\b x_i\right)
\RE\left[\bsigma_i\bsigma_i^{\top}\right]
\right)
\\
&=
\Tr\left(
K\left(\b x_i,\b x_i\right)
\right)
\\
&=
k\left(\b x_i,\b x_i\right)
\Tr\left(\b M\right)
\\
&\le
\kappa\Tr\left(\b M\right).
\end{align*}
Consequently,
\begin{align*}
\RE\left[
\|V_{\bsigma}\|_{\mH_K}^2
\right]
\le
n\kappa\Tr\left(\b M\right).
\end{align*}
Substitution into the first assertion yields
\begin{align*}
\widehat{\mf R}_n^m\left(\mF_{\mT}\left(B\right)\right)
&\le
\frac{BA}{n}
\sqrt{
n\kappa\Tr\left(\b M\right)
}
\\
&=
BA
\sqrt{
\frac{
\kappa\Tr\left(\b M\right)
}{n}
}.
\end{align*}
This proves the second assertion.
\end{proof}

\subsection{Translation and invertible linear composition}

\begin{lemma}[Weighted Sobolev translations]\label{lem:sobolev-translation}
Let $s\ge0$, $d,m\in\Zp$, $\b M\succ0$, and $\b a\in\R^d$. For $\b b\left(\b x\right)=\b x+\b a$, the Koopman operator
\begin{align*}
\mK_{\b b}:
\HMs{s}\left(\R^d,\R^m\right)
\longrightarrow
\HMs{s}\left(\R^d,\R^m\right)
\end{align*}
is an isometry.
\end{lemma}

\begin{proof}
Let $h\in\HMs{s}\left(\R^d,\R^m\right)$. The Fourier transform of a translation satisfies
\begin{align*}
\widehat{h\circ\b b}\left(\bomega\right)
=
e^{\mathrm i\b a^{\top}\bomega}
\widehat h\left(\bomega\right),
\end{align*}
where the sign of the phase depends on the Fourier-transform convention and is irrelevant for the norm. Since the phase is scalar and has modulus one,
\begin{align*}
\widehat{h\circ\b b}\left(\bomega\right)^*
\b M^{-1}
\widehat{h\circ\b b}\left(\bomega\right)
&=
\left|e^{\mathrm i\b a^{\top}\bomega}\right|^2
\widehat h\left(\bomega\right)^*
\b M^{-1}
\widehat h\left(\bomega\right)
\\
&=
\widehat h\left(\bomega\right)^*
\b M^{-1}
\widehat h\left(\bomega\right).
\end{align*}
Therefore,
\begin{align*}
\|\mK_{\b b}h\|_{\HMs{s}\left(\R^d,\R^m\right)}^2
&=
\int_{\R^d}
\left(1+\|\bomega\|_2^2\right)^s
\widehat{h\circ\b b}\left(\bomega\right)^*
\b M^{-1}
\widehat{h\circ\b b}\left(\bomega\right)
\,\d\bomega
\\
&=
\int_{\R^d}
\left(1+\|\bomega\|_2^2\right)^s
\widehat h\left(\bomega\right)^*
\b M^{-1}
\widehat h\left(\bomega\right)
\,\d\bomega
\\
&=
\|h\|_{\HMs{s}\left(\R^d,\R^m\right)}^2.
\end{align*}
Thus $\|\mK_{\b b}h\|=\|h\|$ for every $h$, and the operator is an isometry.
\end{proof}

\begin{lemma}[Invertible weighted Sobolev pullback]\label{lem:invertible-sobolev-pullback}
Let $s_{\mathrm{in}},s_{\mathrm{out}}\ge0$, let $\b M\succ0$, and let $\b W\in\R^{d\times d}$ be invertible. Then
\begin{align*}
\mK_{\b W}:
\HMs{s_{\mathrm{out}}}\left(\R^d,\R^m\right)
\longrightarrow
\HMs{s_{\mathrm{in}}}\left(\R^d,\R^m\right)
\end{align*}
is bounded whenever the displayed supremum below is finite, and
\begin{align*}
\|\mK_{\b W}\|
\le
\sup_{\bomega\in\R^d}
\left(
\frac{
\left(1+\|\b W^{\top}\bomega\|_2^2\right)^{s_{\mathrm{in}}}
}{
\left(1+\|\bomega\|_2^2\right)^{s_{\mathrm{out}}}
}
\right)^{1/2}
\frac{1}{\left|\det\left(\b W\right)\right|^{1/2}}.
\end{align*}
\end{lemma}

\begin{proof}
Let $h\in\HMs{s_{\mathrm{out}}}\left(\R^d,\R^m\right)$. The Fourier scaling formula for an invertible linear map gives
\begin{align*}
\widehat{h\circ\b W}\left(\bomega\right)
=
\frac{1}{\left|\det\left(\b W\right)\right|}
\widehat h\left(\b W^{-\top}\bomega\right).
\end{align*}
Insert this identity into the weighted Sobolev norm:
\begin{align*}
\|\mK_{\b W}h\|_{\HMs{s_{\mathrm{in}}}\left(\R^d,\R^m\right)}^2
&=
\int_{\R^d}
\left(1+\|\bomega\|_2^2\right)^{s_{\mathrm{in}}}
\widehat{h\circ\b W}\left(\bomega\right)^*
\b M^{-1}
\widehat{h\circ\b W}\left(\bomega\right)
\,\d\bomega
\\
&=
\frac{1}{\left|\det\left(\b W\right)\right|^2}
\int_{\R^d}
\left(1+\|\bomega\|_2^2\right)^{s_{\mathrm{in}}}
\widehat h\left(\b W^{-\top}\bomega\right)^*
\b M^{-1}
\widehat h\left(\b W^{-\top}\bomega\right)
\,\d\bomega.
\end{align*}
Use the change of variables
\begin{align*}
\bomega
=
\b W^{\top}\boldsymbol{\xi}.
\end{align*}
Its Jacobian is $\left|\det\left(\b W\right)\right|$, so
\begin{align*}
\d\bomega
=
\left|\det\left(\b W\right)\right|
\d\boldsymbol{\xi}.
\end{align*}
The preceding norm therefore equals
\begin{align*}
\|\mK_{\b W}h\|_{\HMs{s_{\mathrm{in}}}}^2
&=
\frac{1}{\left|\det\left(\b W\right)\right|}
\int_{\R^d}
\left(1+\|\b W^{\top}\boldsymbol{\xi}\|_2^2\right)^{s_{\mathrm{in}}}
\widehat h\left(\boldsymbol{\xi}\right)^*
\b M^{-1}
\widehat h\left(\boldsymbol{\xi}\right)
\,\d\boldsymbol{\xi}
\\
&=
\frac{1}{\left|\det\left(\b W\right)\right|}
\int_{\R^d}
\frac{
\left(1+\|\b W^{\top}\boldsymbol{\xi}\|_2^2\right)^{s_{\mathrm{in}}}
}{
\left(1+\|\boldsymbol{\xi}\|_2^2\right)^{s_{\mathrm{out}}}
}
\left(1+\|\boldsymbol{\xi}\|_2^2\right)^{s_{\mathrm{out}}}
\widehat h\left(\boldsymbol{\xi}\right)^*
\b M^{-1}
\widehat h\left(\boldsymbol{\xi}\right)
\,\d\boldsymbol{\xi}.
\end{align*}
The scalar quotient is nonnegative. Taking its supremum outside the integral yields
\begin{align*}
\|\mK_{\b W}h\|_{\HMs{s_{\mathrm{in}}}}^2
&\le
\frac{1}{\left|\det\left(\b W\right)\right|}
\sup_{\boldsymbol{\xi}\in\R^d}
\frac{
\left(1+\|\b W^{\top}\boldsymbol{\xi}\|_2^2\right)^{s_{\mathrm{in}}}
}{
\left(1+\|\boldsymbol{\xi}\|_2^2\right)^{s_{\mathrm{out}}}
}
\|h\|_{\HMs{s_{\mathrm{out}}}}^2.
\end{align*}
Taking square roots proves the claimed operator-norm estimate.
\end{proof}

\subsection{Proof of the invertible theorem and its corollary}

\begin{proof}[Proof of \Cref{inv}]
For a fixed admissible network, define the complete Koopman chain
\begin{align*}
T
\coloneqq
\mK_{\b W_1}
\mK_{\b b_1}
\mK_{\sigma_1}
\cdots
\mK_{\b W_{L-1}}
\mK_{\b b_{L-1}}
\mK_{\sigma_{L-1}}
\mK_{\b W_L}
\mK_{\b b_L}.
\end{align*}
By \Cref{eq:koopman-factorization}, every function in the class has the form $f=Tg$ with
\begin{align*}
\|g\|_{\HMs{s_L}\left(\R^d,\R^m\right)}
\le B_g.
\end{align*}
Submultiplicativity of the operator norm gives
\begin{align*}
\|T\|
&\le
\prod_{l=1}^{L}
\|\mK_{\b W_l}\|
\prod_{l=1}^{L}
\|\mK_{\b b_l}\|
\prod_{l=1}^{L-1}
\|\mK_{\sigma_l}\|.
\end{align*}
By \Cref{lem:sobolev-translation}, every translation factor is one. By \Cref{lem:invertible-sobolev-pullback}, the $l$th linear factor satisfies
\begin{align*}
\|\mK_{\b W_l}\|
\le
\sup_{\bomega\in\R^d}
\left(
\frac{
\left(1+\|\b W_l^{\top}\bomega\|_2^2\right)^{s_{l-1}}
}{
\left(1+\|\bomega\|_2^2\right)^{s_l}
}
\right)^{1/2}
\frac{1}{\left|\det\left(\b W_l\right)\right|^{1/2}}.
\end{align*}
Consequently,
\begin{align*}
\|T\|
&\le
\prod_{l=1}^{L-1}
\|\mK_{\sigma_l}\|
\\
&\quad\cdot
\prod_{l=1}^{L}
\left[
\sup_{\bomega\in\R^d}
\left(
\frac{
\left(1+\|\b W_l^{\top}\bomega\|_2^2\right)^{s_{l-1}}
}{
\left(1+\|\bomega\|_2^2\right)^{s_l}
}
\right)^{1/2}
\frac{1}{\left|\det\left(\b W_l\right)\right|^{1/2}}
\right].
\end{align*}
Take the supremum of this estimate over all weight tuples in $\W^L$. The input space is the vvRKHS induced by $K_{s_0}=k_{s_0}\b M$, and \Cref{assum:kernel-diagonal} gives
\begin{align*}
k_{s_0}\left(\b x_i,\b x_i\right)
\le\kappa.
\end{align*}
Apply \Cref{lem:operator-image-rademacher} with the terminal Hilbert space $\HMs{s_L}\left(\R^d,\R^m\right)$, radius $B_g$, and the family of complete Koopman chains. The resulting estimate is exactly the bound stated in \Cref{inv}.
\end{proof}

\begin{proof}[Proof of \Cref{cor1}]
Assume $s_l=s$ for every $l$. For any $\b W\in\W$ and any $\bomega\in\R^d$,
\begin{align*}
\|\b W^{\top}\bomega\|_2
\le
\|\b W^{\top}\|
\|\bomega\|_2
=
\|\b W\|
\|\bomega\|_2
\le
C\|\bomega\|_2.
\end{align*}
It follows that
\begin{align*}
1+\|\b W^{\top}\bomega\|_2^2
&\le
1+C^2\|\bomega\|_2^2
\\
&\le
\max\left\{1,C^2\right\}
\left(1+\|\bomega\|_2^2\right).
\end{align*}
Raise both sides to the power $s$, divide by $\left(1+\|\bomega\|_2^2\right)^s$, and take square roots. This gives
\begin{align*}
\left(
\frac{
\left(1+\|\b W^{\top}\bomega\|_2^2\right)^s
}{
\left(1+\|\bomega\|_2^2\right)^s
}
\right)^{1/2}
\le
\max\left\{1,C^s\right\}.
\end{align*}
The determinant condition in the definition of $\W$ gives
\begin{align*}
\frac{1}{\left|\det\left(\b W\right)\right|^{1/2}}
\le
\frac{1}{D^{1/2}}.
\end{align*}
Thus each linear layer contributes at most
\begin{align*}
\frac{\max\left\{1,C^s\right\}}{D^{1/2}}.
\end{align*}
Substituting this uniform estimate into \Cref{inv} proves the corollary.
\end{proof}

\subsection{Injective linear composition and proof of the injective theorem}

\begin{lemma}[Uniform weighted Sobolev restriction]\label{lem:weighted-sobolev-restriction}
Fix integers $d_{\mathrm{in}}\le d_{\mathrm{out}}$ and set $q=d_{\mathrm{out}}-d_{\mathrm{in}}$. Let $s_{\mathrm{out}}>q/2$ and let
\begin{align*}
0\le s_{\mathrm{in}}
\le
s_{\mathrm{out}}-\frac{q}{2}.
\end{align*}
For every $d_{\mathrm{in}}$-dimensional linear subspace $\mS\subset\R^{d_{\mathrm{out}}}$, restriction extends uniquely to a bounded operator
\begin{align*}
R_{\mS}:
\HMs{s_{\mathrm{out}}}\left(\R^{d_{\mathrm{out}}},\R^m\right)
\longrightarrow
\HMs{s_{\mathrm{in}}}\left(\mS,\R^m\right).
\end{align*}
Moreover, the operator norm can be bounded by a constant that depends only on $d_{\mathrm{in}}$, $d_{\mathrm{out}}$, $s_{\mathrm{in}}$, and $s_{\mathrm{out}}$, and not on the orientation of $\mS$ or on the positive-definite matrix $\b M$.
\end{lemma}

\begin{proof}
The scalar Sobolev trace theorem for Bessel-potential spaces states that restriction to a codimension-$q$ linear subspace extends boundedly as
\begin{align*}
H^{s_{\mathrm{out}}}\left(\R^{d_{\mathrm{out}}}\right)
\longrightarrow
H^{s_{\mathrm{out}}-q/2}\left(\mS\right)
\end{align*}
whenever $s_{\mathrm{out}}>q/2$. The continuous Sobolev embedding along the same Euclidean subspace gives
\begin{align*}
H^{s_{\mathrm{out}}-q/2}\left(\mS\right)
\longrightarrow
H^{s_{\mathrm{in}}}\left(\mS\right)
\end{align*}
when $s_{\mathrm{in}}\le s_{\mathrm{out}}-q/2$. Composing these two maps gives a scalar restriction bound
\begin{align*}
\|R_{\mS}u\|_{H^{s_{\mathrm{in}}}\left(\mS\right)}
\le
C_{\mathrm{tr}}
\|u\|_{H^{s_{\mathrm{out}}}\left(\R^{d_{\mathrm{out}}}\right)}.
\end{align*}
Because Euclidean Sobolev norms are invariant under orthogonal changes of coordinates, $C_{\mathrm{tr}}$ may be chosen independently of the orientation of $\mS$.

It remains to pass from the scalar estimate to the $\b M^{-1}$-weighted vector estimate. Let
\begin{align*}
\b A
\coloneqq
\b M^{-1/2}.
\end{align*}
For a vector-valued function $h$, define $u=\b A h$. Since $\b A$ acts only on the output coordinate and restriction acts only on the input variable,
\begin{align*}
\b A\left(R_{\mS}h\right)
=
R_{\mS}\left(\b A h\right).
\end{align*}
The weighted Sobolev norm can be written as the ordinary product Sobolev norm of $\b A h$:
\begin{align*}
\|h\|_{\HMs{s}\left(\R^d,\R^m\right)}^2
=
\sum_{a=1}^{m}
\|\left(\b A h\right)_a\|_{H^s\left(\R^d\right)}^2.
\end{align*}
Apply the scalar restriction estimate to each component of $\b A h$ and sum over $a$:
\begin{align*}
\|R_{\mS}h\|_{\HMs{s_{\mathrm{in}}}\left(\mS,\R^m\right)}^2
&=
\sum_{a=1}^{m}
\|R_{\mS}\left(\b A h\right)_a\|_{H^{s_{\mathrm{in}}}\left(\mS\right)}^2
\\
&\le
C_{\mathrm{tr}}^2
\sum_{a=1}^{m}
\|\left(\b A h\right)_a\|_{H^{s_{\mathrm{out}}}\left(\R^{d_{\mathrm{out}}}\right)}^2
\\
&=
C_{\mathrm{tr}}^2
\|h\|_{\HMs{s_{\mathrm{out}}}\left(\R^{d_{\mathrm{out}}},\R^m\right)}^2.
\end{align*}
The constant is the same scalar trace constant and is therefore independent of $\b M$.
\end{proof}

\begin{lemma}[Injective weighted Sobolev pullback on the range]\label{lem:injective-sobolev-pullback}
Let $d_{\mathrm{in}}\le d_{\mathrm{out}}$, let $s\ge0$, let $\b M\succ0$, and let $\b W\in\R^{d_{\mathrm{out}}\times d_{\mathrm{in}}}$ be injective. Set
\begin{align*}
\mS
\coloneqq
\ra\left(\b W\right).
\end{align*}
Then the pullback
\begin{align*}
P_{\b W}:
\HMs{s}\left(\mS,\R^m\right)
\longrightarrow
\HMs{s}\left(\R^{d_{\mathrm{in}}},\R^m\right),
\qquad
P_{\b W}u
\coloneqq
u\circ\b W,
\end{align*}
is bounded and satisfies
\begin{align*}
\|P_{\b W}\|
\le
\sup_{\bomega\in\mS}
\left(
\frac{
1+\|\b W^{\top}\bomega\|_2^2
}{
1+\|\bomega\|_2^2
}
\right)^{s/2}
\frac{1}{\det\left(\b W^{\top}\b W\right)^{1/4}}.
\end{align*}
\end{lemma}

\begin{proof}
Choose a matrix $\b Q\in\R^{d_{\mathrm{out}}\times d_{\mathrm{in}}}$ whose columns form an orthonormal basis of $\mS$. Thus
\begin{align*}
\b Q^{\top}\b Q
=
\b I_{d_{\mathrm{in}}},
\qquad
\ra\left(\b Q\right)
=
\mS.
\end{align*}
Because every column of $\b W$ belongs to $\mS$, there is a unique square matrix
\begin{align*}
\b A
\coloneqq
\b Q^{\top}\b W
\in
\R^{d_{\mathrm{in}}\times d_{\mathrm{in}}}
\end{align*}
such that
\begin{align*}
\b W
=
\b Q\b A.
\end{align*}
Injectivity of $\b W$ implies invertibility of $\b A$. Indeed, if $\b A\b x=0$, then
\begin{align*}
\b W\b x
=
\b Q\b A\b x
=0,
\end{align*}
and injectivity of $\b W$ forces $\b x=0$.

Identify a function $u:\mS\to\R^m$ with the function
\begin{align*}
\widetilde u:\R^{d_{\mathrm{in}}}\to\R^m,
\qquad
\widetilde u\left(\b z\right)
\coloneqq
u\left(\b Q\b z\right).
\end{align*}
This identification is isometric by the definition of the Sobolev space on $\mS$. Moreover,
\begin{align*}
\left(P_{\b W}u\right)\left(\b x\right)
&=
u\left(\b W\b x\right)
\\
&=
u\left(\b Q\b A\b x\right)
\\
&=
\widetilde u\left(\b A\b x\right).
\end{align*}
Therefore $P_{\b W}$ is, in orthonormal range coordinates, exactly the invertible square pullback associated with $\b A$. Apply \Cref{lem:invertible-sobolev-pullback} with equal input and output order $s$:
\begin{align*}
\|P_{\b W}u\|_{\HMs{s}\left(\R^{d_{\mathrm{in}}},\R^m\right)}
&\le
\sup_{\boldsymbol{\xi}\in\R^{d_{\mathrm{in}}}}
\left(
\frac{
1+\|\b A^{\top}\boldsymbol{\xi}\|_2^2
}{
1+\|\boldsymbol{\xi}\|_2^2
}
\right)^{s/2}
\frac{1}{\left|\det\left(\b A\right)\right|^{1/2}}
\|\widetilde u\|_{\HMs{s}\left(\R^{d_{\mathrm{in}}},\R^m\right)}.
\end{align*}
We now rewrite every factor in coordinate-free form. First,
\begin{align*}
\b W^{\top}\b W
&=
\b A^{\top}\b Q^{\top}\b Q\b A
\\
&=
\b A^{\top}\b A.
\end{align*}
Taking determinants gives
\begin{align*}
\det\left(\b W^{\top}\b W\right)
&=
\det\left(\b A^{\top}\b A\right)
\\
&=
\det\left(\b A\right)^2,
\end{align*}
and hence
\begin{align*}
\left|\det\left(\b A\right)\right|^{1/2}
=
\det\left(\b W^{\top}\b W\right)^{1/4}.
\end{align*}
Second, for $\boldsymbol{\xi}\in\R^{d_{\mathrm{in}}}$, set
\begin{align*}
\bomega
\coloneqq
\b Q\boldsymbol{\xi}
\in\mS.
\end{align*}
Orthonormality gives
\begin{align*}
\|\bomega\|_2
=
\|\boldsymbol{\xi}\|_2,
\end{align*}
and
\begin{align*}
\b W^{\top}\bomega
&=
\b A^{\top}\b Q^{\top}\b Q\boldsymbol{\xi}
\\
&=
\b A^{\top}\boldsymbol{\xi}.
\end{align*}
As $\boldsymbol{\xi}$ ranges over $\R^{d_{\mathrm{in}}}$, $\bomega=\b Q\boldsymbol{\xi}$ ranges over $\mS$. The square-coordinate supremum is therefore equal to
\begin{align*}
\sup_{\bomega\in\mS}
\left(
\frac{
1+\|\b W^{\top}\bomega\|_2^2
}{
1+\|\bomega\|_2^2
}
\right)^{s/2}.
\end{align*}
Finally, the isometric identification gives
\begin{align*}
\|\widetilde u\|_{\HMs{s}\left(\R^{d_{\mathrm{in}}},\R^m\right)}
=
\|u\|_{\HMs{s}\left(\mS,\R^m\right)}.
\end{align*}
Substituting these three identities proves the stated bound.
\end{proof}

\begin{lemma}[Complete injective linear-layer estimate]\label{lem:complete-injective-layer}
Under the trace assumptions of \Cref{inj}, let $\b W_l\in\Wl$. Then
\begin{align*}
\mK_{\b W_l}:
\HMs{s_l}\left(\R^{d_l},\R^m\right)
\longrightarrow
\HMs{s_{l-1}}\left(\R^{d_{l-1}},\R^m\right)
\end{align*}
is bounded and
\begin{align*}
\|\mK_{\b W_l}\|
\le
\mm G_l
\sup_{\bomega\in\ra\left(\b W_l\right)}
\left(
\frac{
1+\|\b W_l^{\top}\bomega\|_2^2
}{
1+\|\bomega\|_2^2
}
\right)^{s_{l-1}/2}
\frac{1}{\det\left(\b W_l^{\top}\b W_l\right)^{1/4}}.
\end{align*}
\end{lemma}

\begin{proof}
Let
\begin{align*}
\mS_l
\coloneqq
\ra\left(\b W_l\right).
\end{align*}
For $h\in\HMs{s_l}\left(\R^{d_l},\R^m\right)$, composition with $\b W_l$ depends only on the restriction of $h$ to $\mS_l$. More precisely,
\begin{align*}
h\circ\b W_l
=
\left(h|_{\mS_l}\right)\circ\b W_l.
\end{align*}
Thus the full Koopman operator factors as
\begin{align*}
\mK_{\b W_l}
=
P_{\b W_l}R_{\b W_l},
\end{align*}
where $R_{\b W_l}$ is the restriction map defined in the main text and $P_{\b W_l}$ is the range pullback from \Cref{lem:injective-sobolev-pullback}. Submultiplicativity gives
\begin{align*}
\|\mK_{\b W_l}\|
\le
\|P_{\b W_l}\|
\|R_{\b W_l}\|.
\end{align*}
By definition of $\mm G_l$,
\begin{align*}
\|R_{\b W_l}\|
\le
\mm G_l.
\end{align*}
Apply \Cref{lem:injective-sobolev-pullback} with $s=s_{l-1}$ to the other factor. Multiplying the two estimates proves the result.
\end{proof}

\begin{proof}[Proof of \Cref{inj}]
For a fixed injective network, let $T$ denote the complete Koopman chain in \Cref{eq:koopman-factorization}. As in the proof of \Cref{inv}, submultiplicativity gives
\begin{align*}
\|T\|
&\le
\prod_{l=1}^{L}
\|\mK_{\b W_l}\|
\prod_{l=1}^{L}
\|\mK_{\b b_l}\|
\prod_{l=1}^{L-1}
\|\mK_{\sigma_l}\|.
\end{align*}
The dimensions of a translation agree on its domain and codomain, so \Cref{lem:sobolev-translation} applies at every layer and gives
\begin{align*}
\|\mK_{\b b_l}\|
=1.
\end{align*}
For the linear layers, \Cref{lem:complete-injective-layer} gives
\begin{align*}
\|\mK_{\b W_l}\|
\le
\mm G_l
\sup_{\bomega\in\ra\left(\b W_l\right)}
\left(
\frac{
1+\|\b W_l^{\top}\bomega\|_2^2
}{
1+\|\bomega\|_2^2
}
\right)^{s_{l-1}/2}
\frac{1}{\det\left(\b W_l^{\top}\b W_l\right)^{1/4}}.
\end{align*}
Consequently,
\begin{align*}
\|T\|
&\le
\prod_{l=1}^{L-1}
\|\mK_{\sigma_l}\|
\\
&\quad\cdot
\prod_{l=1}^{L}
\left[
\mm G_l
\sup_{\bomega\in\ra\left(\b W_l\right)}
\left(
\frac{
1+\|\b W_l^{\top}\bomega\|_2^2
}{
1+\|\bomega\|_2^2
}
\right)^{s_{l-1}/2}
\frac{1}{\det\left(\b W_l^{\top}\b W_l\right)^{1/4}}
\right].
\end{align*}
Take the supremum over the product class $\prod_{l=1}^{L}\Wl$. Every function in $\F_{\scriptscriptstyle\mm{inj}}$ has the form $Tg$ with terminal norm at most $B_g$. The input vvRKHS has kernel $K_{s_0}=k_{s_0}\b M$, and its scalar diagonal is bounded by $\kappa$ under \Cref{assum:kernel-diagonal}. Applying \Cref{lem:operator-image-rademacher} proves the theorem.
\end{proof}

\section{Proof of the Brownian/Cameron--Martin Bound}\label{app:brownian-proof}

\subsection{Exact vector-valued Brownian RKHS}

\begin{lemma}[Signed-interval feature representation]\label{lem:signed-interval-feature}
For $x\in\left[-R,R\right]$, define $\psi_x\in L^2\left(\left[-R,R\right]\right)$ by
\begin{align*}
\psi_x\left(t\right)
\coloneqq
\begin{cases}
\mathbf 1_{\left(0,x\right]}\left(t\right),
& x\ge0,
\\
-\mathbf 1_{\left[x,0\right)}\left(t\right),
& x<0.
\end{cases}
\end{align*}
Then, for all $x,x'\in\left[-R,R\right]$,
\begin{align*}
\left\langle
\psi_x,
\psi_{x'}
\right\rangle_{L^2\left(\left[-R,R\right]\right)}
=
\kB\left(x,x'\right).
\end{align*}
\end{lemma}

\begin{proof}
There are three sign configurations.

\proofstep{Both points are nonnegative.}
If $x\ge0$ and $x'\ge0$, then the supports are the intervals $\left(0,x\right]$ and $\left(0,x'\right]$. Their intersection has length $\min\left\{x,x'\right\}$, so
\begin{align*}
\left\langle\psi_x,\psi_{x'}\right\rangle
=
\min\left\{x,x'\right\}.
\end{align*}
For nonnegative $x$ and $x'$,
\begin{align*}
\frac{|x|+|x'|-|x-x'|}{2}
&=
\frac{x+x'-|x-x'|}{2}
\\
&=
\min\left\{x,x'\right\}.
\end{align*}
Thus the desired identity holds.

\proofstep{Both points are nonpositive.}
If $x<0$ and $x'<0$, both feature functions carry a minus sign, so their product is positive on the overlap of $\left[x,0\right)$ and $\left[x',0\right)$. The overlap has length $\min\left\{|x|,|x'|\right\}$, and hence
\begin{align*}
\left\langle\psi_x,\psi_{x'}\right\rangle
=
\min\left\{|x|,|x'|\right\}.
\end{align*}
Because $x$ and $x'$ have the same negative sign,
\begin{align*}
|x-x'|
=
\left||x|-|x'|\right|.
\end{align*}
Therefore,
\begin{align*}
\frac{|x|+|x'|-|x-x'|}{2}
=
\min\left\{|x|,|x'|\right\}.
\end{align*}
Again the desired identity holds.

\proofstep{The points have opposite signs.}
If one point is nonnegative and the other is negative, the two feature functions have disjoint supports up to the single point zero, which has Lebesgue measure zero. Hence their inner product is zero. In the same sign configuration,
\begin{align*}
|x-x'|
=
|x|+|x'|,
\end{align*}
so $\kB\left(x,x'\right)=0$. This completes all cases.
\end{proof}

\begin{lemma}[Exact vector-valued Cameron--Martin characterization]\label{lem:BrownianCameronMartin}
Let $\mX=\left[-R,R\right]$, let $\b M\succ0$, and let
\begin{align*}
K\left(x,x'\right)
=
\kB\left(x,x'\right)\b M.
\end{align*}
The vvRKHS $\HB\left(\mX,\R^m\right)$ is exactly
\begin{align*}
\left\{
f:\mX\to\R^m
\;\middle|\;
f\left(0\right)=0,
\ f\text{ is absolutely continuous},
\ f'\in L^2\left(\mX,\R^m\right)
\right\},
\end{align*}
and its inner product and norm are
\begin{align*}
\left\langle f,h\right\rangle_{\HB}
&=
\int_{-R}^{R}
f'\left(t\right)^{\top}
\b M^{-1}
h'\left(t\right)
\,\d t,
\\
\|f\|_{\HB}^2
&=
\int_{-R}^{R}
f'\left(t\right)^{\top}
\b M^{-1}
f'\left(t\right)
\,\d t.
\end{align*}
\end{lemma}

\begin{proof}
Define the weighted derivative space
\begin{align*}
L_{\b M}^2\left(\mX,\R^m\right)
\coloneqq
\left\{
 v:\mX\to\R^m
 \;\middle|\;
 \int_{-R}^{R}
 v\left(t\right)^{\top}\b M^{-1}v\left(t\right)
 \,\d t
 <\infty
\right\},
\end{align*}
with inner product
\begin{align*}
\left\langle v,w\right\rangle_{L_{\b M}^2}
\coloneqq
\int_{-R}^{R}
v\left(t\right)^{\top}
\b M^{-1}
w\left(t\right)
\,\d t.
\end{align*}
Let $\lambda_{\min}\left(\b M\right)>0$ and $\lambda_{\max}\left(\b M\right)>0$ denote the smallest and largest eigenvalues of $\b M$. For every $z\in\R^m$,
\begin{align*}
\frac{1}{\lambda_{\max}\left(\b M\right)}
\|z\|_2^2
\le
z^{\top}\b M^{-1}z
\le
\frac{1}{\lambda_{\min}\left(\b M\right)}
\|z\|_2^2.
\end{align*}
The weighted norm is therefore equivalent to the ordinary vector-valued $L^2$ norm. Since $L^2\left(\mX,\R^m\right)$ is complete, $L_{\b M}^2\left(\mX,\R^m\right)$ is a Hilbert space.

Let $\mC_{\b M}$ denote the anchored absolutely continuous space displayed in the lemma. Define
\begin{align*}
J:
L_{\b M}^2\left(\mX,\R^m\right)
&\longrightarrow
\mC_{\b M},
\\
\left(Jv\right)\left(x\right)
&\coloneqq
\int_{0}^{x}v\left(t\right)\,\d t,
\qquad
x\in\left[-R,R\right],
\end{align*}
where, for $x<0$, the oriented integral means
\begin{align*}
\int_0^x v\left(t\right)\,\d t
=
-
\int_x^0 v\left(t\right)\,\d t.
\end{align*}
For every $v\in L_{\b M}^2$, the function $Jv$ is absolutely continuous, satisfies $\left(Jv\right)\left(0\right)=0$, and has weak derivative
\begin{align*}
\left(Jv\right)'
=
v
\end{align*}
almost everywhere. Thus $Jv\in\mC_{\b M}$. Conversely, if $f\in\mC_{\b M}$, the fundamental theorem for absolutely continuous functions and the anchoring condition give
\begin{align*}
f\left(x\right)
&=
f\left(0\right)
+
\int_0^x f'\left(t\right)\,\d t
\\
&=
\int_0^x f'\left(t\right)\,\d t
\\
&=
\left(Jf'\right)\left(x\right)
\end{align*}
for every $x\in\mX$. Hence $J$ is surjective. It is injective because $Jv=0$ implies $v=\left(Jv\right)'=0$ almost everywhere. Finally,
\begin{align*}
\|Jv\|_{\mC_{\b M}}^2
&=
\int_{-R}^{R}
\left(Jv\right)'\left(t\right)^{\top}
\b M^{-1}
\left(Jv\right)'\left(t\right)
\,\d t
\\
&=
\int_{-R}^{R}
v\left(t\right)^{\top}
\b M^{-1}
v\left(t\right)
\,\d t
\\
&=
\|v\|_{L_{\b M}^2}^2.
\end{align*}
Therefore $J$ is an isometric bijection from the Hilbert space $L_{\b M}^2\left(\mX,\R^m\right)$ onto $\mC_{\b M}$. In particular, $\mC_{\b M}$ is complete.

Fix $x\in\mX$ and $\b y\in\R^m$. The kernel section is
\begin{align*}
K\left(\cdot,x\right)\b y
=
\kB\left(\cdot,x\right)\b M\b y.
\end{align*}
By \Cref{lem:signed-interval-feature}, the weak derivative of the scalar function $t\mapsto\kB\left(t,x\right)$ is $\psi_x\left(t\right)$ for almost every $t$. Consequently,
\begin{align*}
\frac{\d}{\d t}
\left(
K\left(t,x\right)\b y
\right)
=
\psi_x\left(t\right)
\b M\b y
\end{align*}
for almost every $t$. This derivative belongs to $L^2$, and the kernel section vanishes at $t=0$. Hence $K\left(\cdot,x\right)\b y\in\mC_{\b M}$.

For $f\in\mC_{\b M}$, compute its inner product with the kernel section:
\begin{align*}
\left\langle
f,
K\left(\cdot,x\right)\b y
\right\rangle_{\mC_{\b M}}
&=
\int_{-R}^{R}
f'\left(t\right)^{\top}
\b M^{-1}
\left(
\psi_x\left(t\right)\b M\b y
\right)
\,\d t
\\
&=
\int_{-R}^{R}
\psi_x\left(t\right)
f'\left(t\right)^{\top}
\b y
\,\d t
\\
&=
\left(
\int_{-R}^{R}
\psi_x\left(t\right)
f'\left(t\right)
\,\d t
\right)^{\top}
\b y.
\end{align*}
The signed-interval definition of $\psi_x$ and the fundamental theorem for absolutely continuous functions give
\begin{align*}
\int_{-R}^{R}
\psi_x\left(t\right)
f'\left(t\right)
\,\d t
=
f\left(x\right)-f\left(0\right)
=
f\left(x\right).
\end{align*}
Therefore,
\begin{align*}
\left\langle
f,
K\left(\cdot,x\right)\b y
\right\rangle_{\mC_{\b M}}
=
f\left(x\right)^{\top}\b y.
\end{align*}
Thus $K$ is a reproducing kernel for $\mC_{\b M}$. By uniqueness of the RKHS associated with a positive semidefinite kernel, $\mC_{\b M}=\HB\left(\mX,\R^m\right)$ isometrically. The displayed inner product and norm are therefore exact.
\end{proof}

\subsection{Brownian composition operators}

\begin{lemma}[Scalar linear maps in the Brownian RKHS]\label{lem:brownian-linear-map}
Let $W\in\R$ satisfy $|W|\le1$. Then
\begin{align*}
\mK_W:
\HB\left(\left[-R,R\right],\R^m\right)
\longrightarrow
\HB\left(\left[-R,R\right],\R^m\right)
\end{align*}
is bounded and
\begin{align*}
\|\mK_W\|
\le
|W|^{1/2}.
\end{align*}
\end{lemma}

\begin{proof}
Let $f\in\HB\left(\left[-R,R\right],\R^m\right)$. Since $|W|\le1$, the point $Wx$ belongs to $\left[-R,R\right]$ whenever $x\in\left[-R,R\right]$. Also,
\begin{align*}
\left(f\circ W\right)\left(0\right)
=
f\left(0\right)
=0.
\end{align*}
The chain rule for absolutely continuous functions gives
\begin{align*}
\left(f\circ W\right)'\left(x\right)
=
W f'\left(Wx\right)
\end{align*}
for almost every $x$. By \Cref{lem:BrownianCameronMartin},
\begin{align*}
\|\mK_Wf\|_{\HB}^2
&=
\int_{-R}^{R}
\left(Wf'\left(Wx\right)\right)^{\top}
\b M^{-1}
\left(Wf'\left(Wx\right)\right)
\,\d x
\\
&=
|W|^2
\int_{-R}^{R}
f'\left(Wx\right)^{\top}
\b M^{-1}
f'\left(Wx\right)
\,\d x.
\end{align*}
If $W=0$, the last expression is zero and the result is immediate. Assume $W\ne0$ and set
\begin{align*}
Q\left(y\right)
\coloneqq
f'\left(y\right)^{\top}
\b M^{-1}
f'\left(y\right).
\end{align*}
Because $\b M^{-1}\succ0$, one has $Q\left(y\right)\ge0$ for almost every $y$.

First suppose that $W>0$. With $y=Wx$ and $\d x=\d y/W$,
\begin{align*}
|W|^2
\int_{-R}^{R}
Q\left(Wx\right)
\,\d x
&=
W^2
\frac{1}{W}
\int_{-WR}^{WR}
Q\left(y\right)
\,\d y
\\
&=
|W|
\int_{-|W|R}^{|W|R}
Q\left(y\right)
\,\d y.
\end{align*}
Next suppose that $W<0$. The same substitution reverses the integration limits, and therefore
\begin{align*}
|W|^2
\int_{-R}^{R}
Q\left(Wx\right)
\,\d x
&=
W^2
\frac{1}{W}
\int_{-WR}^{WR}
Q\left(y\right)
\,\d y
\\
&=
W
\int_{|W|R}^{-|W|R}
Q\left(y\right)
\,\d y
\\
&=
-W
\int_{-|W|R}^{|W|R}
Q\left(y\right)
\,\d y
\\
&=
|W|
\int_{-|W|R}^{|W|R}
Q\left(y\right)
\,\d y.
\end{align*}
Thus, in both cases,
\begin{align*}
\|\mK_Wf\|_{\HB}^2
&=
|W|
\int_{-|W|R}^{|W|R}
Q\left(y\right)
\,\d y
\\
&\le
|W|
\int_{-R}^{R}
Q\left(y\right)
\,\d y
\\
&=
|W|
\|f\|_{\HB}^2,
\end{align*}
where the inequality uses $|W|\le1$ and $Q\ge0$. Taking square roots proves the operator-norm bound.
\end{proof}

\begin{lemma}[Anchored diffeomorphisms in the Brownian RKHS]\label{lem:brownian-activation-map}
Let $\sigma:\left[-R,R\right]\to\left[-R,R\right]$ be a $C^1$ diffeomorphism satisfying $\sigma\left(0\right)=0$. Then
\begin{align*}
\mK_\sigma:
\HB\left(\left[-R,R\right],\R^m\right)
\longrightarrow
\HB\left(\left[-R,R\right],\R^m\right)
\end{align*}
is bounded and
\begin{align*}
\|\mK_\sigma\|
\le
\|\sigma'\|_\infty^{1/2}.
\end{align*}
\end{lemma}

\begin{proof}
Let $f\in\HB\left(\left[-R,R\right],\R^m\right)$. The domain and anchor conditions give
\begin{align*}
\left(f\circ\sigma\right)\left(0\right)
=
f\left(\sigma\left(0\right)\right)
=
f\left(0\right)
=0.
\end{align*}
The chain rule gives
\begin{align*}
\left(f\circ\sigma\right)'\left(x\right)
=
f'\left(\sigma\left(x\right)\right)
\sigma'\left(x\right)
\end{align*}
for almost every $x$. Therefore,
\begin{align*}
\|\mK_\sigma f\|_{\HB}^2
&=
\int_{-R}^{R}
|\sigma'\left(x\right)|^2
f'\left(\sigma\left(x\right)\right)^{\top}
\b M^{-1}
f'\left(\sigma\left(x\right)\right)
\,\d x.
\end{align*}
Because $\sigma$ and $\sigma^{-1}$ are both $C^1$, the chain rule applied to $\sigma^{-1}\circ\sigma=\operatorname{id}$ gives
\begin{align*}
\left(\sigma^{-1}\right)'\left(\sigma\left(x\right)\right)
\sigma'\left(x\right)
=1
\end{align*}
for every $x\in\left[-R,R\right]$. Hence $\sigma'\left(x\right)\ne0$ everywhere. Continuity of $\sigma'$ on the connected interval then implies that $\sigma'$ has a constant sign, so $\sigma$ is strictly monotone. Use the change of variables $y=\sigma\left(x\right)$. Since $\sigma$ maps the interval onto itself,
\begin{align*}
\d x
=
\left|
\left(\sigma^{-1}\right)'\left(y\right)
\right|
\d y
=
\frac{1}{
\left|
\sigma'\left(\sigma^{-1}\left(y\right)\right)
\right|
}
\d y.
\end{align*}
Substitution yields
\begin{align*}
\|\mK_\sigma f\|_{\HB}^2
&=
\int_{-R}^{R}
\left|
\sigma'\left(\sigma^{-1}\left(y\right)\right)
\right|
f'\left(y\right)^{\top}
\b M^{-1}
f'\left(y\right)
\,\d y
\\
&\le
\|\sigma'\|_\infty
\int_{-R}^{R}
f'\left(y\right)^{\top}
\b M^{-1}
f'\left(y\right)
\,\d y
\\
&=
\|\sigma'\|_\infty
\|f\|_{\HB}^2.
\end{align*}
Taking square roots proves the claim. Notice that the inverse derivative cancels exactly in the change of variables; it is needed only to justify the diffeomorphic change of variables, not as a multiplicative norm factor.
\end{proof}

\subsection{Proof of the Brownian theorem}

\begin{proof}[Proof of \Cref{thm:BrownianKoopman}]
For a fixed admissible Brownian network, define
\begin{align*}
T
\coloneqq
\mK_{W_1}
\mK_{\sigma_1}
\cdots
\mK_{\sigma_{L-1}}
\mK_{W_L}.
\end{align*}
Every function in $\F_{\mathrm{inv}}^{\scriptscriptstyle\left(\mathrm B\right)}$ has the form $Tg$ with $\|g\|_{\HB}\le B_g$. Submultiplicativity, \Cref{lem:brownian-linear-map}, and \Cref{lem:brownian-activation-map} give
\begin{align*}
\|T\|
&\le
\prod_{l=1}^{L}
\|\mK_{W_l}\|
\prod_{l=1}^{L-1}
\|\mK_{\sigma_l}\|
\\
&\le
\prod_{l=1}^{L}
|W_l|^{1/2}
\prod_{l=1}^{L-1}
\|\sigma_l'\|_\infty^{1/2}.
\end{align*}
Take the supremum over $\left(W_1,\ldots,W_L\right)\in\W_{\mathrm B}^L$.

The input kernel is $K\left(x,x'\right)=\kB\left(x,x'\right)\b M$. On the diagonal,
\begin{align*}
\kB\left(x,x\right)
&=
\frac{|x|+|x|-|x-x|}{2}
\\
&=
|x|
\\
&\le
R.
\end{align*}
Apply \Cref{lem:operator-image-rademacher} with $\kappa=R$, terminal radius $B_g$, and the family of complete Brownian Koopman chains. This gives
\begin{align*}
\eR\left(\F_{\mathrm{inv}}^{\scriptscriptstyle\left(\mathrm B\right)}\right)
&\le
B_g
\sqrt{
\frac{
R\Tr\left(\b M\right)
}{n}
}
\prod_{l=1}^{L-1}
\|\sigma_l'\|_\infty^{1/2}
\\
&\quad\cdot
\sup_{\left(W_1,\ldots,W_L\right)\in\W_{\mathrm B}^L}
\prod_{l=1}^{L}
|W_l|^{1/2},
\end{align*}
which is the theorem.
\end{proof}

\section{Proofs for Shared Operator Learning and Transfer}\label{app:shared-transfer-proofs}

\subsection{Hilbert--Schmidt rank-one identities}

\begin{lemma}[Rank-one Hilbert--Schmidt identities]\label{lem:rank-one-hs}
Let $\mH_1$ and $\mH_2$ be Hilbert spaces. For $u,u'\in\mH_2$ and $v,v'\in\mH_1$, define
\begin{align*}
\left(u\otimes v\right)h
=
\left\langle v,h\right\rangle_{\mH_1}u.
\end{align*}
Then $u\otimes v$ is Hilbert--Schmidt and
\begin{align*}
\|u\otimes v\|_{\mathrm{HS}}
&=
\|u\|_{\mH_2}
\|v\|_{\mH_1},
\\
\left\langle
u\otimes v,
u'\otimes v'
\right\rangle_{\mathrm{HS}}
&=
\left\langle u,u'\right\rangle_{\mH_2}
\left\langle v,v'\right\rangle_{\mH_1}.
\end{align*}
Moreover, for every Hilbert--Schmidt operator $T:\mH_1\to\mH_2$,
\begin{align*}
\left\langle
T,
u\otimes v
\right\rangle_{\mathrm{HS}}
=
\left\langle
Tv,
u
\right\rangle_{\mH_2}.
\end{align*}
\end{lemma}

\begin{proof}
Let $\left(e_j\right)_{j\in J}$ be an orthonormal basis of $\mH_1$. By the definition of the Hilbert--Schmidt norm,
\begin{align*}
\|u\otimes v\|_{\mathrm{HS}}^2
&=
\sum_{j\in J}
\|\left(u\otimes v\right)e_j\|_{\mH_2}^2
\\
&=
\sum_{j\in J}
\left|
\left\langle v,e_j\right\rangle_{\mH_1}
\right|^2
\|u\|_{\mH_2}^2
\\
&=
\|v\|_{\mH_1}^2
\|u\|_{\mH_2}^2,
\end{align*}
where the last equality is Parseval's identity. This proves the first formula.

For the inner product of two rank-one operators,
\begin{align*}
\left\langle
u\otimes v,
u'\otimes v'
\right\rangle_{\mathrm{HS}}
&=
\sum_{j\in J}
\left\langle
\left(u\otimes v\right)e_j,
\left(u'\otimes v'\right)e_j
\right\rangle_{\mH_2}
\\
&=
\sum_{j\in J}
\left\langle v,e_j\right\rangle_{\mH_1}
\overline{
\left\langle v',e_j\right\rangle_{\mH_1}
}
\left\langle u,u'\right\rangle_{\mH_2}
\\
&=
\left\langle v,v'\right\rangle_{\mH_1}
\left\langle u,u'\right\rangle_{\mH_2}.
\end{align*}
All spaces in the paper are real, so the complex conjugation may be omitted; it is displayed only to emphasize the standard Hilbert-space identity.

Finally,
\begin{align*}
\left\langle
T,
u\otimes v
\right\rangle_{\mathrm{HS}}
&=
\sum_{j\in J}
\left\langle
Te_j,
\left(u\otimes v\right)e_j
\right\rangle_{\mH_2}
\\
&=
\sum_{j\in J}
\left\langle v,e_j\right\rangle_{\mH_1}
\left\langle Te_j,u\right\rangle_{\mH_2}
\\
&=
\left\langle
T\left(
\sum_{j\in J}
\left\langle v,e_j\right\rangle_{\mH_1}e_j
\right),
u
\right\rangle_{\mH_2}
\\
&=
\left\langle Tv,u\right\rangle_{\mH_2},
\end{align*}
where convergence follows from the Hilbert--Schmidt property and Parseval's identity. This proves all three formulas.
\end{proof}

\subsection{Well-posedness of the shared operator problem}

\begin{proof}[Proof of \Cref{prop:shared-transfer-well-posed}]
Set
\begin{align*}
\mE
\coloneqq
\m L_2\left(\mH_{\mathrm{out}},\mH_K\right).
\end{align*}
This is a Hilbert space under the Hilbert--Schmidt inner product. Define the finite-dimensional prediction space
\begin{align*}
\mY_{\mathrm{tr}}
\coloneqq
\prod_{t=1}^{N_{\mathrm s}}
\prod_{i=1}^{n_t}
\R^m
\end{align*}
with norm
\begin{align*}
\|z\|_{\mY_{\mathrm{tr}}}^2
\coloneqq
\sum_{t=1}^{N_{\mathrm s}}
\sum_{i=1}^{n_t}
\|z_{ti}\|_2^2.
\end{align*}
Define the linear prediction map
\begin{align*}
\Phi:\mE
&\longrightarrow
\mY_{\mathrm{tr}},
\\
\left(\Phi T\right)_{ti}
&\coloneqq
\left(Tg_t\right)\left(\b x_{ti}\right).
\end{align*}
We first prove that $\Phi$ is bounded. Fix $T\in\mE$, $t\in\left[N_{\mathrm s}\right]$, and $i\in\left[n_t\right]$. By Euclidean duality and the reproducing property,
\begin{align*}
\left\|\left(Tg_t\right)\left(\b x_{ti}\right)\right\|_2
&=
\sup_{\|\b y\|_2\le1}
\left|
\b y^{\top}
\left(Tg_t\right)\left(\b x_{ti}\right)
\right|
\\
&=
\sup_{\|\b y\|_2\le1}
\left|
\left\langle
Tg_t,
K\left(\cdot,\b x_{ti}\right)\b y
\right\rangle_{\mH_K}
\right|
\\
&\le
\|Tg_t\|_{\mH_K}
\sup_{\|\b y\|_2\le1}
\left\|K\left(\cdot,\b x_{ti}\right)\b y\right\|_{\mH_K}
\\
&\le
\|T\|
\|g_t\|_{\mH_{\mathrm{out}}}
\sup_{\|\b y\|_2\le1}
\left\|K\left(\cdot,\b x_{ti}\right)\b y\right\|_{\mH_K}.
\end{align*}
For every $\b y\in\R^m$, a second use of the reproducing property gives
\begin{align*}
\left\|K\left(\cdot,\b x_{ti}\right)\b y\right\|_{\mH_K}^2
&=
\b y^{\top}
K\left(\b x_{ti},\b x_{ti}\right)
\b y
\\
&\le
\left\|K\left(\b x_{ti},\b x_{ti}\right)\right\|_{\mathrm{op}}
\|\b y\|_2^2.
\end{align*}
Since $\|T\|\le\|T\|_{\mathrm{HS}}$, it follows that
\begin{align*}
\left\|\left(Tg_t\right)\left(\b x_{ti}\right)\right\|_2^2
\le
\left\|K\left(\b x_{ti},\b x_{ti}\right)\right\|_{\mathrm{op}}
\|g_t\|_{\mH_{\mathrm{out}}}^2
\|T\|_{\mathrm{HS}}^2.
\end{align*}
Define the finite constant
\begin{align*}
C_{\Phi}^2
\coloneqq
\sum_{t=1}^{N_{\mathrm s}}
\sum_{i=1}^{n_t}
\left\|K\left(\b x_{ti},\b x_{ti}\right)\right\|_{\mathrm{op}}
\|g_t\|_{\mH_{\mathrm{out}}}^2.
\end{align*}
Summing the preceding pointwise estimate yields
\begin{align*}
\|\Phi T\|_{\mY_{\mathrm{tr}}}
\le
C_{\Phi}
\|T\|_{\mathrm{HS}}.
\end{align*}
Thus $\Phi$ is bounded.

Define the finite-dimensional loss functional
\begin{align*}
\mL\left(z\right)
\coloneqq
\sum_{t=1}^{N_{\mathrm s}}
\sum_{i=1}^{n_t}
\ell\left(\b y_{ti},z_{ti}\right),
\qquad
z\in\mY_{\mathrm{tr}},
\end{align*}
and define the objective
\begin{align*}
J\left(T\right)
\coloneqq
\mL\left(\Phi T\right)
+
\lambda\|T\|_{\mathrm{HS}}^2.
\end{align*}
Because $\ell$ is finite-valued, continuous, convex in its second argument, and nonnegative, $\mL$ is finite-valued, continuous, convex, and nonnegative on the finite-dimensional space $\mY_{\mathrm{tr}}$. Consequently,
\begin{align*}
J\left(T\right)
\ge
\lambda\|T\|_{\mathrm{HS}}^2,
\qquad
T\in\mE.
\end{align*}
Let
\begin{align*}
j_\star
\coloneqq
\inf_{T\in\mE}J\left(T\right).
\end{align*}
The value $j_\star$ is finite because $0\le j_\star\le J\left(0\right)<\infty$. Choose a minimizing sequence $\left(T_r\right)_{r\in\Zp}$ satisfying
\begin{align*}
J\left(T_r\right)
\le
j_\star+
\frac{1}{r}.
\end{align*}
The coercive lower bound gives
\begin{align*}
\lambda\|T_r\|_{\mathrm{HS}}^2
\le
J\left(T_r\right)
\le
j_\star+1,
\end{align*}
so the sequence is bounded in the Hilbert space $\mE$. Hilbert spaces are reflexive; therefore, after passing to a subsequence without changing notation, there is $T_\star\in\mE$ such that
\begin{align*}
T_r
\rightharpoonup
T_\star
\end{align*}
weakly in $\mE$. Bounded linear maps preserve weak convergence, so
\begin{align*}
\Phi T_r
\rightharpoonup
\Phi T_\star
\end{align*}
weakly in $\mY_{\mathrm{tr}}$. Since $\mY_{\mathrm{tr}}$ is finite-dimensional, weak convergence and norm convergence coincide. Hence
\begin{align*}
\Phi T_r
\longrightarrow
\Phi T_\star
\end{align*}
in $\mY_{\mathrm{tr}}$. Continuity of $\mL$ gives
\begin{align*}
\mL\left(\Phi T_r\right)
\longrightarrow
\mL\left(\Phi T_\star\right).
\end{align*}
The Hilbert--Schmidt norm is weakly lower semicontinuous, so
\begin{align*}
\|T_\star\|_{\mathrm{HS}}^2
\le
\liminf_{r\to\infty}
\|T_r\|_{\mathrm{HS}}^2.
\end{align*}
Combining the last two relations gives
\begin{align*}
J\left(T_\star\right)
&=
\mL\left(\Phi T_\star\right)
+
\lambda\|T_\star\|_{\mathrm{HS}}^2
\\
&\le
\liminf_{r\to\infty}
\left[
\mL\left(\Phi T_r\right)
+
\lambda\|T_r\|_{\mathrm{HS}}^2
\right]
\\
&=
j_\star.
\end{align*}
By the definition of $j_\star$, the reverse inequality $j_\star\le J\left(T_\star\right)$ holds. Thus $J\left(T_\star\right)=j_\star$, and a minimizer exists.

It remains to prove uniqueness. Let $S,T\in\mE$ with $S\ne T$, and let $\theta\in\left(0,1\right)$. Convexity of $\mL$ and linearity of $\Phi$ give
\begin{align*}
\mL\left(\Phi\left(\theta S+\left(1-\theta\right)T\right)\right)
\le
\theta\mL\left(\Phi S\right)
+
\left(1-\theta\right)\mL\left(\Phi T\right).
\end{align*}
The Hilbert-space norm identity gives
\begin{align*}
\left\|\theta S+\left(1-\theta\right)T\right\|_{\mathrm{HS}}^2
&=
\theta\|S\|_{\mathrm{HS}}^2
+
\left(1-\theta\right)\|T\|_{\mathrm{HS}}^2
\\
&\quad-
\theta\left(1-\theta\right)
\|S-T\|_{\mathrm{HS}}^2.
\end{align*}
Because $\lambda>0$, $\theta\left(1-\theta\right)>0$, and $S\ne T$, the final term is strictly negative. Therefore,
\begin{align*}
J\left(\theta S+\left(1-\theta\right)T\right)
<
\theta J\left(S\right)
+
\left(1-\theta\right)J\left(T\right).
\end{align*}
Thus $J$ is strictly convex and can have at most one minimizer. Existence and uniqueness together prove the proposition.
\end{proof}

\subsection{Proof of the shared-operator representer theorem}

\begin{proof}[Proof of \Cref{thm:shared-transfer-representer}]
For $t\in\left[N_{\mathrm s}\right]$, $i\in\left[n_t\right]$, and $a\in\left[m\right]$, define the Hilbert--Schmidt operator
\begin{align*}
\Xi_{tia}
\coloneqq
\left(
K\left(\cdot,\b x_{ti}\right)\bea
\right)
\otimes g_t.
\end{align*}
Let
\begin{align*}
\mV
\coloneqq
\Span\left\{
\Xi_{tia}:
1\le t\le N_{\mathrm s},
\ 1\le i\le n_t,
\ 1\le a\le m
\right\}
\end{align*}
inside the Hilbert space $\m L_2\left(\mH_{\mathrm{out}},\mH_K\right)$. The span is finite-dimensional and therefore closed.

Fix an arbitrary Hilbert--Schmidt operator $T$. Decompose it orthogonally as
\begin{align*}
T
=
T_{\mV}+T_{\mV^\perp},
\qquad
T_{\mV}\in\mV,
\qquad
T_{\mV^\perp}\in\mV^\perp.
\end{align*}
We first show that every training prediction depends only on $T_{\mV}$. For a fixed triple $\left(t,i,a\right)$, the reproducing property gives
\begin{align*}
\bea^{\top}
\left(Tg_t\right)\left(\b x_{ti}\right)
&=
\left\langle
Tg_t,
K\left(\cdot,\b x_{ti}\right)\bea
\right\rangle_{\mH_K}.
\end{align*}
Apply the final identity of \Cref{lem:rank-one-hs} with
\begin{align*}
u
=
K\left(\cdot,\b x_{ti}\right)\bea,
\qquad
v
=
g_t.
\end{align*}
It gives
\begin{align*}
\bea^{\top}
\left(Tg_t\right)\left(\b x_{ti}\right)
&=
\left\langle
T,
\Xi_{tia}
\right\rangle_{\mathrm{HS}}
\\
&=
\left\langle
T_{\mV},
\Xi_{tia}
\right\rangle_{\mathrm{HS}}
+
\left\langle
T_{\mV^\perp},
\Xi_{tia}
\right\rangle_{\mathrm{HS}}.
\end{align*}
Since $\Xi_{tia}\in\mV$ and $T_{\mV^\perp}\perp\mV$, the last inner product is zero. Hence
\begin{align*}
\bea^{\top}
\left(Tg_t\right)\left(\b x_{ti}\right)
=
\bea^{\top}
\left(T_{\mV}g_t\right)\left(\b x_{ti}\right)
\end{align*}
for every coordinate $a$. Therefore,
\begin{align*}
\left(Tg_t\right)\left(\b x_{ti}\right)
=
\left(T_{\mV}g_t\right)\left(\b x_{ti}\right)
\end{align*}
for every training pair.

The loss part of the objective in \Cref{eq:shared-transfer-opt} is consequently unchanged when $T$ is replaced by $T_{\mV}$. Orthogonality gives the Pythagorean identity
\begin{align*}
\|T\|_{\mathrm{HS}}^2
=
\|T_{\mV}\|_{\mathrm{HS}}^2
+
\|T_{\mV^\perp}\|_{\mathrm{HS}}^2.
\end{align*}
Because $\lambda>0$, the regularization term is strictly smaller after deleting a nonzero orthogonal component. Hence every minimizer must satisfy
\begin{align*}
T_{\mV^\perp}
=0.
\end{align*}
Thus every minimizer belongs to $\mV$. By the definition of $\mV$, there are coefficients $c_{tia}\in\R$ such that
\begin{align*}
\widehat T
=
\sum_{t=1}^{N_{\mathrm s}}
\sum_{i=1}^{n_t}
\sum_{a=1}^{m}
c_{tia}
\left(
K\left(\cdot,\b x_{ti}\right)\bea
\right)
\otimes g_t.
\end{align*}
The coefficients need not be unique if the spanning operators are linearly dependent, but the represented operator is the minimizer.

To obtain the prediction formula, let $g\in\mH_{\mathrm{out}}$. By the definition of a rank-one operator,
\begin{align*}
\widehat T g
&=
\sum_{t=1}^{N_{\mathrm s}}
\sum_{i=1}^{n_t}
\sum_{a=1}^{m}
c_{tia}
\left\langle g_t,g\right\rangle_{\mH_{\mathrm{out}}}
K\left(\cdot,\b x_{ti}\right)\bea.
\end{align*}
Evaluating at $\b x\in\mX$ yields
\begin{align*}
\left(\widehat T g\right)\left(\b x\right)
&=
\sum_{t=1}^{N_{\mathrm s}}
\sum_{i=1}^{n_t}
\sum_{a=1}^{m}
c_{tia}
\left\langle g_t,g\right\rangle_{\mH_{\mathrm{out}}}
K\left(\b x,\b x_{ti}\right)\bea,
\end{align*}
which proves the theorem.
\end{proof}

\subsection{Proof of the finite-dimensional reduction}

\begin{proof}[Proof of \Cref{thm:shared-transfer-finite}]
Let a finite-rank operator be represented as
\begin{align*}
T_c
=
\sum_{t'=1}^{N_{\mathrm s}}
\sum_{j=1}^{n_{t'}}
\sum_{a=1}^{m}
c_{t'ja}
\left(
K\left(\cdot,\b x_{t'j}\right)\bea
\right)
\otimes g_{t'}.
\end{align*}
For source task $t$ and sample point $\b x_{ti}$, apply the definition of the rank-one operators:
\begin{align*}
T_cg_t
&=
\sum_{t'=1}^{N_{\mathrm s}}
\sum_{j=1}^{n_{t'}}
\sum_{a=1}^{m}
c_{t'ja}
\left\langle g_{t'},g_t\right\rangle_{\mH_{\mathrm{out}}}
K\left(\cdot,\b x_{t'j}\right)\bea.
\end{align*}
Evaluation at $\b x_{ti}$ gives
\begin{align*}
\left(T_cg_t\right)\left(\b x_{ti}\right)
&=
\sum_{t'=1}^{N_{\mathrm s}}
\sum_{j=1}^{n_{t'}}
\sum_{a=1}^{m}
c_{t'ja}
\left\langle g_{t'},g_t\right\rangle_{\mH_{\mathrm{out}}}
K\left(\b x_{ti},\b x_{t'j}\right)\bea.
\end{align*}
Under squared loss, substituting this expression gives the first term in the coefficient objective stated in the theorem.

It remains to calculate the Hilbert--Schmidt norm. Expand the squared norm using bilinearity:
\begin{align*}
\|T_c\|_{\mathrm{HS}}^2
&=
\sum_{t=1}^{N_{\mathrm s}}
\sum_{i=1}^{n_t}
\sum_{a=1}^{m}
\sum_{t'=1}^{N_{\mathrm s}}
\sum_{j=1}^{n_{t'}}
\sum_{b=1}^{m}
c_{tia}c_{t'jb}
\\
&\quad\cdot
\left\langle
\left(
K\left(\cdot,\b x_{ti}\right)\bea
\right)\otimes g_t,
\left(
K\left(\cdot,\b x_{t'j}\right)\beb
\right)\otimes g_{t'}
\right\rangle_{\mathrm{HS}}.
\end{align*}
Apply the rank-one inner-product identity from \Cref{lem:rank-one-hs}:
\begin{align*}
\|T_c\|_{\mathrm{HS}}^2
&=
\sum_{t=1}^{N_{\mathrm s}}
\sum_{i=1}^{n_t}
\sum_{a=1}^{m}
\sum_{t'=1}^{N_{\mathrm s}}
\sum_{j=1}^{n_{t'}}
\sum_{b=1}^{m}
c_{tia}c_{t'jb}
\left\langle g_t,g_{t'}\right\rangle_{\mH_{\mathrm{out}}}
\\
&\quad\cdot
\left\langle
K\left(\cdot,\b x_{ti}\right)\bea,
K\left(\cdot,\b x_{t'j}\right)\beb
\right\rangle_{\mH_K}.
\end{align*}
The vvRKHS reproducing property gives
\begin{align*}
\left\langle
K\left(\cdot,\b x_{ti}\right)\bea,
K\left(\cdot,\b x_{t'j}\right)\beb
\right\rangle_{\mH_K}
=
\bea^{\top}
K\left(\b x_{ti},\b x_{t'j}\right)
\beb.
\end{align*}
Therefore,
\begin{align*}
\|T_c\|_{\mathrm{HS}}^2
&=
\sum_{t=1}^{N_{\mathrm s}}
\sum_{i=1}^{n_t}
\sum_{a=1}^{m}
\sum_{t'=1}^{N_{\mathrm s}}
\sum_{j=1}^{n_{t'}}
\sum_{b=1}^{m}
c_{tia}c_{t'jb}
\left\langle g_t,g_{t'}\right\rangle_{\mH_{\mathrm{out}}}
\bea^{\top}
K\left(\b x_{ti},\b x_{t'j}\right)
\beb.
\end{align*}
Multiplication by $\lambda$ gives the regularization term in the theorem.

By \Cref{thm:shared-transfer-representer}, every minimizer of the infinite-dimensional problem has at least one coefficient representation of the displayed form. Conversely, every coefficient tensor defines a Hilbert--Schmidt finite-rank operator $T_c$. The prediction and norm calculations above show that the value of the infinite-dimensional objective at $T_c$ is exactly the displayed coefficient objective. Hence the two optimization problems have the same feasible objective values on the representer subspace, and the representer theorem shows that this subspace contains every minimizer. Their optimal values are therefore equal, and the coefficients of a minimizer solve the stated finite-dimensional problem.
\end{proof}

\subsection{Rademacher complexity of a fixed operator image}

\begin{lemma}[Conditional target-class complexity]\label{lem:conditional-target-complexity}
Let $K\left(\b x,\b x'\right)=k\left(\b x,\b x'\right)\b M$ with $\b M\succeq0$ and $k\left(\b x,\b x\right)\le\kappa$. Fix a bounded operator $T:\mH_{\mathrm{out}}\to\mH_K$ and define
\begin{align*}
\F_T\left(B\right)
=
\left\{
Tg:
\|g\|_{\mH_{\mathrm{out}}}\le B
\right\}.
\end{align*}
For every sample of size $n_0$,
\begin{align*}
\widehat{\mf R}_{n_0}^m\left(\F_T\left(B\right)\right)
\le
B\|T\|
\sqrt{
\frac{
\kappa\Tr\left(\b M\right)
}{n_0}
}.
\end{align*}
The same bound holds with $\|T\|_{\mathrm{HS}}$ in place of $\|T\|$.
\end{lemma}

\begin{proof}
Apply \Cref{lem:operator-image-rademacher} to the singleton operator family $\mT=\left\{T\right\}$. This gives the estimate with $\|T\|$. Since every Hilbert--Schmidt operator satisfies
\begin{align*}
\|T\|
\le
\|T\|_{\mathrm{HS}},
\end{align*}
the second estimate follows immediately.
\end{proof}

\subsection{A two-sided expected-Rademacher deviation lemma}

\begin{lemma}[Two-sided expected-Rademacher deviation]\label{lem:two-sided-expected-rademacher}
Let $Z_1,\ldots,Z_n$ be independent and identically distributed random variables with values in a measurable space $\mZ$, let $S=\left(Z_1,\ldots,Z_n\right)$, and let $\mG$ be a class of measurable functions $g:\mZ\to\left[0,C\right]$. Assume that the suprema below are measurable. Define
\begin{align*}
Pg
&\coloneqq
\mathbb E\left[g\left(Z_1\right)\right],
\\
P_Sg
&\coloneqq
\frac{1}{n}
\sum_{i=1}^{n}
g\left(Z_i\right),
\\
\widehat{\mf R}_S\left(\mG\right)
&\coloneqq
\mathbb E_{\varepsilon}
\left[
\sup_{g\in\mG}
\frac{1}{n}
\sum_{i=1}^{n}
\varepsilon_i g\left(Z_i\right)
\right],
\\
\mf R_n\left(\mG\right)
&\coloneqq
\mathbb E_S
\left[
\widehat{\mf R}_S\left(\mG\right)
\right],
\end{align*}
where $\varepsilon_1,\ldots,\varepsilon_n$ are independent Rademacher variables independent of $S$. Then, with probability at least $1-\delta$, the following inequalities hold simultaneously for every $g\in\mG$:
\begin{align*}
Pg
&\le
P_Sg
+
2\mf R_n\left(\mG\right)
+
C
\sqrt{
\frac{
\log\left(2/\delta\right)
}{2n}
},
\\
P_Sg
&\le
Pg
+
2\mf R_n\left(\mG\right)
+
C
\sqrt{
\frac{
\log\left(2/\delta\right)
}{2n}
}.
\end{align*}
\end{lemma}

\begin{proof}
Define
\begin{align*}
\Phi_+\left(S\right)
&\coloneqq
\sup_{g\in\mG}
\left(
Pg-P_Sg
\right),
\\
\Phi_-\left(S\right)
&\coloneqq
\sup_{g\in\mG}
\left(
P_Sg-Pg
\right).
\end{align*}
We first bound the expectations of these two random variables.

Let $S'=\left(Z_1',\ldots,Z_n'\right)$ be an independent copy of $S$. Since $Pg=\mathbb E_{S'}\left[P_{S'}g\right]$, Jensen's inequality for the supremum gives
\begin{align*}
\mathbb E_S\left[\Phi_+\left(S\right)\right]
&=
\mathbb E_S
\left[
\sup_{g\in\mG}
\left(
\mathbb E_{S'}\left[P_{S'}g\right]-P_Sg
\right)
\right]
\\
&\le
\mathbb E_{S,S'}
\left[
\sup_{g\in\mG}
\left(
P_{S'}g-P_Sg
\right)
\right]
\\
&=
\mathbb E_{S,S'}
\left[
\sup_{g\in\mG}
\frac{1}{n}
\sum_{i=1}^{n}
\left(
 g\left(Z_i'\right)-g\left(Z_i\right)
\right)
\right].
\end{align*}
For every sign vector $\varepsilon=\left(\varepsilon_1,\ldots,\varepsilon_n\right)$, swap $Z_i$ and $Z_i'$ whenever $\varepsilon_i=-1$. Because each pair $\left(Z_i,Z_i'\right)$ is exchangeable and the pairs are independent, the joint distribution is unchanged. Therefore,
\begin{align*}
&\mathbb E_{S,S'}
\left[
\sup_{g\in\mG}
\frac{1}{n}
\sum_{i=1}^{n}
\left(
 g\left(Z_i'\right)-g\left(Z_i\right)
\right)
\right]
\\
&\qquad=
\mathbb E_{S,S',\varepsilon}
\left[
\sup_{g\in\mG}
\frac{1}{n}
\sum_{i=1}^{n}
\varepsilon_i
\left(
 g\left(Z_i'\right)-g\left(Z_i\right)
\right)
\right].
\end{align*}
For every fixed realization of $S$, $S'$, and $\varepsilon$,
\begin{align*}
&\sup_{g\in\mG}
\frac{1}{n}
\sum_{i=1}^{n}
\varepsilon_i
\left(
 g\left(Z_i'\right)-g\left(Z_i\right)
\right)
\\
&\qquad\le
\sup_{g\in\mG}
\frac{1}{n}
\sum_{i=1}^{n}
\varepsilon_i g\left(Z_i'\right)
+
\sup_{g\in\mG}
\frac{1}{n}
\sum_{i=1}^{n}
\left(-\varepsilon_i\right)g\left(Z_i\right).
\end{align*}
Taking expectations, using that $S'$ has the same distribution as $S$, and using that $-\varepsilon$ has the same distribution as $\varepsilon$, we obtain
\begin{align*}
\mathbb E_S\left[\Phi_+\left(S\right)\right]
\le
2\mf R_n\left(\mG\right).
\end{align*}

The reverse deviation is handled explicitly in the same way. Since $Pg=\mathbb E_{S'}\left[P_{S'}g\right]$,
\begin{align*}
\mathbb E_S\left[\Phi_-\left(S\right)\right]
&=
\mathbb E_S
\left[
\sup_{g\in\mG}
\left(
P_Sg-\mathbb E_{S'}\left[P_{S'}g\right]
\right)
\right]
\\
&\le
\mathbb E_{S,S'}
\left[
\sup_{g\in\mG}
\left(
P_Sg-P_{S'}g
\right)
\right].
\end{align*}
Interchanging the names of $S$ and $S'$ shows that the final expectation is equal to the ghost-sample expectation already bounded above. Hence
\begin{align*}
\mathbb E_S\left[\Phi_-\left(S\right)\right]
\le
2\mf R_n\left(\mG\right).
\end{align*}

We next apply bounded differences. If one coordinate $Z_i$ is replaced by an arbitrary value $\widetilde Z_i$, then, for every $g\in\mG$,
\begin{align*}
\left|
P_Sg-P_{\widetilde S}g
\right|
&=
\frac{1}{n}
\left|
 g\left(Z_i\right)-g\left(\widetilde Z_i\right)
\right|
\\
&\le
\frac{C}{n}.
\end{align*}
Taking suprema shows that changing one sample coordinate changes each of $\Phi_+\left(S\right)$ and $\Phi_-\left(S\right)$ by at most $C/n$. McDiarmid's inequality therefore gives
\begin{align*}
\P\left(
\Phi_+\left(S\right)
>
\mathbb E\left[\Phi_+\left(S\right)\right]
+
C
\sqrt{
\frac{
\log\left(2/\delta\right)
}{2n}
}
\right)
&\le
\frac{\delta}{2},
\\
\P\left(
\Phi_-\left(S\right)
>
\mathbb E\left[\Phi_-\left(S\right)\right]
+
C
\sqrt{
\frac{
\log\left(2/\delta\right)
}{2n}
}
\right)
&\le
\frac{\delta}{2}.
\end{align*}
A union bound, together with the two expectation estimates, implies that with probability at least $1-\delta$,
\begin{align*}
\Phi_+\left(S\right)
&\le
2\mf R_n\left(\mG\right)
+
C
\sqrt{
\frac{
\log\left(2/\delta\right)
}{2n}
},
\\
\Phi_-\left(S\right)
&\le
2\mf R_n\left(\mG\right)
+
C
\sqrt{
\frac{
\log\left(2/\delta\right)
}{2n}
}.
\end{align*}
Unpacking the definitions of $\Phi_+$ and $\Phi_-$ gives the two simultaneous inequalities in the lemma.
\end{proof}

\subsection{Proof of the conditional target-transfer bound}

\begin{proof}[Proof of \Cref{prop:shared-transfer-bound}]
Condition on the source data and hence on $\widehat T$. Under the independence assumption, the target observations remain independent and identically distributed after conditioning, while $\widehat T$ is fixed. We may therefore carry out the complete argument conditionally for the fixed class
\begin{align*}
\F_{\widehat T}\left(B\right)
=
\left\{
\widehat T g:
\|g\|_{\mH_{\mathrm{out}}}\le B
\right\}.
\end{align*}

Let
\begin{align*}
S_0
\coloneqq
\left(
\left(X_1,Y_1\right),
\ldots,
\left(X_{n_0},Y_{n_0}\right)
\right)
\end{align*}
denote the target sample. Define the empirical target risk by
\begin{align*}
\widehat R_0\left(f\right)
\coloneqq
\frac{1}{n_0}
\sum_{i=1}^{n_0}
\ell\left(Y_i,f\left(X_i\right)\right),
\end{align*}
and define the scalar loss class
\begin{align*}
\mG_{\widehat T}
\coloneqq
\left\{
\left(x,y\right)
\longmapsto
\ell\left(y,f\left(x\right)\right):
f\in\F_{\widehat T}\left(B\right)
\right\}.
\end{align*}
For a fixed realization of $S_0$, its empirical scalar Rademacher complexity is
\begin{align*}
\widehat{\mf R}_{S_0}\left(\mG_{\widehat T}\right)
\coloneqq
\mathbb E_{\varepsilon}
\left[
\sup_{f\in\F_{\widehat T}\left(B\right)}
\frac{1}{n_0}
\sum_{i=1}^{n_0}
\varepsilon_i
\ell\left(Y_i,f\left(X_i\right)\right)
\right].
\end{align*}
For each $i$, set $c_i\coloneqq\ell\left(Y_i,0\right)$. Because $c_i$ is independent of $f$, for every fixed Rademacher vector,
\begin{align*}
&\sup_{f\in\F_{\widehat T}\left(B\right)}
\frac{1}{n_0}
\sum_{i=1}^{n_0}
\varepsilon_i
\ell\left(Y_i,f\left(X_i\right)\right)
\\
&\qquad=
\frac{1}{n_0}
\sum_{i=1}^{n_0}
\varepsilon_i c_i
+
\sup_{f\in\F_{\widehat T}\left(B\right)}
\frac{1}{n_0}
\sum_{i=1}^{n_0}
\varepsilon_i
\left[
\ell\left(Y_i,f\left(X_i\right)\right)-c_i
\right].
\end{align*}
Taking the Rademacher expectation and using
\begin{align*}
\mathbb E_{\varepsilon}
\left[
\frac{1}{n_0}
\sum_{i=1}^{n_0}
\varepsilon_i c_i
\right]
=0
\end{align*}
shows that subtracting $\ell\left(Y_i,0\right)$ leaves the empirical Rademacher complexity unchanged.

For each fixed observation $\left(X_i,Y_i\right)$, the map
\begin{align*}
\b z
\longmapsto
\ell\left(Y_i,\b z\right)-\ell\left(Y_i,0\right)
\end{align*}
is $L_\ell$-Lipschitz with respect to the Euclidean norm and vanishes at the origin. The vector-contraction inequality of \citet{maurer2016vectorcontraction} therefore gives
\begin{align*}
\widehat{\mf R}_{S_0}\left(\mG_{\widehat T}\right)
\le
\sqrt{2}\,L_\ell
\widehat{\mf R}_{n_0}^m
\left(
\F_{\widehat T}\left(B\right)
\right).
\end{align*}
The scalar complexity on the left uses the conventional one-sided supremum. The absolute-value vector complexity used in this paper is at least as large as its one-sided counterpart, so the displayed inequality remains valid with our vector-complexity definition on the right.

Define the conditional complexity radius
\begin{align*}
r_{\widehat T}
\coloneqq
\sqrt{2}\,L_\ell B
\|\widehat T\|
\sqrt{
\frac{
\kappa\Tr\left(\b M\right)
}{n_0}
}.
\end{align*}
By \Cref{lem:conditional-target-complexity}, for every target sample realization,
\begin{align*}
\widehat{\mf R}_{S_0}\left(\mG_{\widehat T}\right)
\le
r_{\widehat T}.
\end{align*}
Consequently, the conditional expected Rademacher complexity satisfies
\begin{align*}
\mf R_{n_0}\left(\mG_{\widehat T}\mid\widehat T\right)
&\coloneqq
\mathbb E_{S_0\mid\widehat T}
\left[
\widehat{\mf R}_{S_0}\left(\mG_{\widehat T}\right)
\right]
\\
&\le
r_{\widehat T}.
\end{align*}
The quantity $r_{\widehat T}$ is deterministic after conditioning on the source data.

Apply \Cref{lem:two-sided-expected-rademacher} conditionally on $\widehat T$ to the class $\mG_{\widehat T}\subset\left[0,C_\ell\right]$. With conditional probability at least $1-\delta$, the following inequalities hold simultaneously for every $f\in\F_{\widehat T}\left(B\right)$:
\begin{align*}
R_0\left(f\right)
&\le
\widehat R_0\left(f\right)
+
2r_{\widehat T}
+
C_\ell
\sqrt{
\frac{
\log\left(2/\delta\right)
}{2n_0}
},
\\
\widehat R_0\left(f\right)
&\le
R_0\left(f\right)
+
2r_{\widehat T}
+
C_\ell
\sqrt{
\frac{
\log\left(2/\delta\right)
}{2n_0}
}.
\end{align*}

Let $\eta>0$. By the definition of the infimum, there is $f_\eta\in\F_{\widehat T}\left(B\right)$ such that
\begin{align*}
R_0\left(f_\eta\right)
\le
\inf_{f\in\F_{\widehat T}\left(B\right)}
R_0\left(f\right)
+
\eta.
\end{align*}
Because $\widehat f_0$ is an empirical risk minimizer,
\begin{align*}
\widehat R_0\left(\widehat f_0\right)
\le
\widehat R_0\left(f_\eta\right).
\end{align*}
Apply the first deviation inequality to $\widehat f_0$, use empirical optimality, and then apply the second deviation inequality to $f_\eta$:
\begin{align*}
R_0\left(\widehat f_0\right)
&\le
\widehat R_0\left(\widehat f_0\right)
+
2r_{\widehat T}
+
C_\ell
\sqrt{
\frac{
\log\left(2/\delta\right)
}{2n_0}
}
\\
&\le
\widehat R_0\left(f_\eta\right)
+
2r_{\widehat T}
+
C_\ell
\sqrt{
\frac{
\log\left(2/\delta\right)
}{2n_0}
}
\\
&\le
R_0\left(f_\eta\right)
+
4r_{\widehat T}
+
2C_\ell
\sqrt{
\frac{
\log\left(2/\delta\right)
}{2n_0}
}
\\
&\le
\inf_{f\in\F_{\widehat T}\left(B\right)}
R_0\left(f\right)
+
\eta
+
4r_{\widehat T}
+
2C_\ell
\sqrt{
\frac{
\log\left(2/\delta\right)
}{2n_0}
}.
\end{align*}
Let $\eta\downarrow0$ and subtract $R_0^\star$. By the definition of $A_{\widehat T}\left(B\right)$,
\begin{align*}
R_0\left(\widehat f_0\right)-R_0^\star
\le
A_{\widehat T}\left(B\right)
+
4r_{\widehat T}
+
2C_\ell
\sqrt{
\frac{
\log\left(2/\delta\right)
}{2n_0}
}.
\end{align*}
Insert the definition of $r_{\widehat T}$:
\begin{align*}
R_0\left(\widehat f_0\right)-R_0^\star
&\le
A_{\widehat T}\left(B\right)
+
4\sqrt{2}\,L_\ell B
\|\widehat T\|
\sqrt{
\frac{
\kappa\Tr\left(\b M\right)
}{n_0}
}
\\
&\quad+
2C_\ell
\sqrt{
\frac{
\log\left(2/\delta\right)
}{2n_0}
}.
\end{align*}
This is the asserted conditional bound. Replacing $\|\widehat T\|$ by $\|\widehat T\|_{\mathrm{HS}}$ is valid because every Hilbert--Schmidt operator satisfies $\|\widehat T\|\le\|\widehat T\|_{\mathrm{HS}}$.
\end{proof}

\section{A Sufficient Sobolev Activation Condition}\label{app:activation-condition}

The main Sobolev theorems assume boundedness of each activation Koopman operator. The following lemma records a standard sufficient condition in the integer-order case and verifies that the same scalar composition estimate remains valid for the $\b M^{-1}$-weighted vector norm.

\begin{lemma}[Bounded Sobolev activation composition]\label{lem:bounded-koopman-activation}
Let $s\in\Zp$, let $d,m\in\Zp$, and let $\b M\succ0$. Suppose that $\sigma:\R^d\to\R^d$ is a $C^s$ bi-Lipschitz diffeomorphism, that every partial derivative $D^\alpha\sigma_j$ with $1\le|\alpha|\le s$ is bounded, and that
\begin{align*}
\left\|\det D\sigma^{-1}\right\|_{L^\infty\left(\R^d\right)}
<\infty.
\end{align*}
Then
\begin{align*}
\mK_\sigma:
\HMs{s}\left(\R^d,\R^m\right)
\longrightarrow
\HMs{s}\left(\R^d,\R^m\right)
\end{align*}
is bounded.
\end{lemma}

\begin{proof}
We first establish the scalar estimate on the dense core $C_c^\infty\left(\R^d\right)$. For integer $s$, the Fourier Sobolev norm is equivalent to the weak-derivative norm. Thus there are constants $c_1,c_2>0$, depending only on $d$, $s$, and the Fourier normalization, such that every scalar $u\in H^s\left(\R^d\right)$ satisfies
\begin{align*}
c_1
\sum_{|\alpha|\le s}
\|D^\alpha u\|_{L^2\left(\R^d\right)}^2
&\le
\|u\|_{H^s\left(\R^d\right)}^2
\\
&\le
c_2
\sum_{|\alpha|\le s}
\|D^\alpha u\|_{L^2\left(\R^d\right)}^2.
\end{align*}
Fix $u\in C_c^\infty\left(\R^d\right)$. We bound every derivative of $u\circ\sigma$ of order at most $s$.

For the zeroth-order term, the change of variables $y=\sigma\left(x\right)$ gives
\begin{align*}
\|u\circ\sigma\|_{L^2\left(\R^d\right)}^2
&=
\int_{\R^d}
|u\left(\sigma\left(x\right)\right)|^2
\,\d x
\\
&=
\int_{\R^d}
|u\left(y\right)|^2
\left|
\det D\sigma^{-1}\left(y\right)
\right|
\,\d y
\\
&\le
\left\|\det D\sigma^{-1}\right\|_{L^\infty}
\|u\|_{L^2\left(\R^d\right)}^2.
\end{align*}

Fix a multi-index $\alpha$ with $1\le|\alpha|\le s$. Because both $u$ and $\sigma$ are classically differentiable to the required orders, the multivariate chain rule expresses $D^\alpha\left(u\circ\sigma\right)$ as a finite sum of terms of the form
\begin{align*}
\left(D^\beta u\right)\left(\sigma\left(x\right)\right)
P_{\alpha,\beta}
\left(
\left(D^\gamma\sigma_j\left(x\right)\right)_{
1\le j\le d,
\ 1\le|\gamma|\le|\alpha|
}
\right),
\end{align*}
where $1\le|\beta|\le|\alpha|$ and $P_{\alpha,\beta}$ is a polynomial determined only by the multi-indices. The number of terms is finite and depends only on $d$ and $s$. By the assumed bounds on the derivatives of $\sigma$, every polynomial factor is uniformly bounded. Hence there is $C_\alpha>0$ such that
\begin{align*}
\left|
D^\alpha\left(u\circ\sigma\right)\left(x\right)
\right|^2
\le
C_\alpha
\sum_{1\le|\beta|\le|\alpha|}
\left|
\left(D^\beta u\right)\left(\sigma\left(x\right)\right)
\right|^2.
\end{align*}
Integrate over $x$ and apply the same change of variables to every term:
\begin{align*}
\|D^\alpha\left(u\circ\sigma\right)\|_{L^2}^2
&\le
C_\alpha
\sum_{1\le|\beta|\le|\alpha|}
\int_{\R^d}
\left|
\left(D^\beta u\right)\left(\sigma\left(x\right)\right)
\right|^2
\,\d x
\\
&=
C_\alpha
\sum_{1\le|\beta|\le|\alpha|}
\int_{\R^d}
|D^\beta u\left(y\right)|^2
\left|
\det D\sigma^{-1}\left(y\right)
\right|
\,\d y
\\
&\le
C_\alpha
\left\|\det D\sigma^{-1}\right\|_{L^\infty}
\sum_{1\le|\beta|\le|\alpha|}
\|D^\beta u\|_{L^2}^2.
\end{align*}
Sum over all $\alpha$ with $|\alpha|\le s$. Since only finitely many multi-indices occur, there is $C_\sigma>0$ such that
\begin{align*}
\sum_{|\alpha|\le s}
\|D^\alpha\left(u\circ\sigma\right)\|_{L^2}^2
\le
C_\sigma
\sum_{|\beta|\le s}
\|D^\beta u\|_{L^2}^2.
\end{align*}
Use the norm equivalence in both directions:
\begin{align*}
\|u\circ\sigma\|_{H^s\left(\R^d\right)}^2
&\le
c_2
\sum_{|\alpha|\le s}
\|D^\alpha\left(u\circ\sigma\right)\|_{L^2}^2
\\
&\le
c_2C_\sigma
\sum_{|\beta|\le s}
\|D^\beta u\|_{L^2}^2
\\
&\le
\frac{c_2C_\sigma}{c_1}
\|u\|_{H^s\left(\R^d\right)}^2.
\end{align*}
Define
\begin{align*}
C_{\mathrm{comp}}
\coloneqq
\left(
\frac{c_2C_\sigma}{c_1}
\right)^{1/2}.
\end{align*}
We have proved
\begin{align*}
\|u\circ\sigma\|_{H^s}
\le
C_{\mathrm{comp}}
\|u\|_{H^s}
\end{align*}
for every $u\in C_c^\infty\left(\R^d\right)$.

We now extend the estimate to an arbitrary $u\in H^s\left(\R^d\right)$. Choose $u_r\in C_c^\infty\left(\R^d\right)$ such that
\begin{align*}
u_r
\longrightarrow
u
\end{align*}
in $H^s\left(\R^d\right)$. For every $r,q\in\Zp$, the estimate on the smooth core applied to $u_r-u_q$ gives
\begin{align*}
\|u_r\circ\sigma-u_q\circ\sigma\|_{H^s}
&=
\|\left(u_r-u_q\right)\circ\sigma\|_{H^s}
\\
&\le
C_{\mathrm{comp}}
\|u_r-u_q\|_{H^s}.
\end{align*}
Thus $\left(u_r\circ\sigma\right)_{r\in\Zp}$ is Cauchy in $H^s\left(\R^d\right)$. Completeness gives a function $v\in H^s\left(\R^d\right)$ such that
\begin{align*}
u_r\circ\sigma
\longrightarrow
v
\end{align*}
in $H^s\left(\R^d\right)$ and therefore in $L^2\left(\R^d\right)$.

The same change of variables used for the zeroth-order term gives
\begin{align*}
\|\left(u_r-u\right)\circ\sigma\|_{L^2\left(\R^d\right)}^2
&=
\int_{\R^d}
|u_r\left(y\right)-u\left(y\right)|^2
\left|
\det D\sigma^{-1}\left(y\right)
\right|
\,\d y
\\
&\le
\left\|\det D\sigma^{-1}\right\|_{L^\infty}
\|u_r-u\|_{L^2\left(\R^d\right)}^2
\\
&\longrightarrow
0.
\end{align*}
Hence $u_r\circ\sigma$ converges to $u\circ\sigma$ in $L^2$. The $L^2$ limit is unique, so $v=u\circ\sigma$ almost everywhere. Therefore $u\circ\sigma\in H^s\left(\R^d\right)$. Passing to the limit in the smooth-core estimate gives
\begin{align*}
\|u\circ\sigma\|_{H^s}
&=
\lim_{r\to\infty}
\|u_r\circ\sigma\|_{H^s}
\\
&\le
C_{\mathrm{comp}}
\lim_{r\to\infty}
\|u_r\|_{H^s}
\\
&=
C_{\mathrm{comp}}
\|u\|_{H^s}.
\end{align*}
Thus scalar composition is bounded on all of $H^s\left(\R^d\right)$.

We now pass to the weighted vector-valued space. Let
\begin{align*}
\b A
\coloneqq
\b M^{-1/2}
\end{align*}
and write $v=\b A f$. Since $\b A$ is a constant output transformation,
\begin{align*}
\b A\left(f\circ\sigma\right)
=
\left(\b A f\right)\circ\sigma
=
v\circ\sigma.
\end{align*}
The weighted norm is
\begin{align*}
\|f\|_{\HMs{s}\left(\R^d,\R^m\right)}^2
=
\sum_{a=1}^{m}
\|v_a\|_{H^s\left(\R^d\right)}^2.
\end{align*}
Apply the scalar estimate to every component $v_a$:
\begin{align*}
\|f\circ\sigma\|_{\HMs{s}}^2
&=
\sum_{a=1}^{m}
\|v_a\circ\sigma\|_{H^s}^2
\\
&\le
C_{\mathrm{comp}}^2
\sum_{a=1}^{m}
\|v_a\|_{H^s}^2
\\
&=
C_{\mathrm{comp}}^2
\|f\|_{\HMs{s}}^2.
\end{align*}
Taking square roots proves boundedness of $\mK_\sigma$ on the weighted vector-valued Sobolev space.
\end{proof}

\section{Additional Experimental Details}\label{app:numerics}

\begingroup
\small
\subsection{Proxy regularization on MNIST}

The regularization experiment uses the stabilized numerical proxies described in \Cref{sec:experiments}. These quantities are optimization penalties motivated by the layer geometry of the theory; they are not evaluations of \Cref{inv,inj,thm:BrownianKoopman} for the rank-deficient experimental architecture.

The fully connected network has the form
\begin{align*}
f_{\theta}\left(x\right)
&=
g\left(
W_4
\sigma\left(
W_3
\sigma\left(
W_2
\sigma\left(
W_1x+b_1
\right)
+b_2
\right)
+b_3
\right)
+b_4
\right),
\end{align*}
with weight dimensions
\begin{align*}
W_1&\in\R^{1024\times784},
&
W_2&\in\R^{2048\times1024},
\\
W_3&\in\R^{2048\times2048},
&
W_4&\in\R^{10\times2048}.
\end{align*}
The bias dimensions are $b_1\in\R^{1024}$, $b_2\in\R^{2048}$, $b_3\in\R^{2048}$, and $b_4\in\R^{10}$. Thus the hidden widths are $1024$, $2048$, and $2048$, followed by a ten-dimensional output layer.

The matrices $W_1$ and $W_2$ use orthogonal initialization~\citep{saxe2014exact}; $W_3$ and $W_4$ use a truncated-normal initialization. All biases use a uniform initialization. The activation is the smooth Leaky ReLU of \citet{biswas2022smooth}. Training uses $1000$ randomly selected MNIST training examples, runs for $1800$ epochs, and uses Adam~\citep{kingma2017adammethodstochasticoptimization} with learning rate $10^{-4}$. No additional $L_2$ weight decay is used.

Let $\mI=\left\{1,2\right\}$ denote the layers on which the proxy penalties are evaluated. The Sobolev-inspired and Brownian-inspired penalties are, respectively,
\begin{align*}
\mathrm{SR}
&\coloneqq
\sum_{l\in\mI}
\frac{
\|W_l\|^{s_l}
}{
\det\left(
\b I+W_l^{\top}W_l
\right)^{1/4}
},
\qquad
s_l
\coloneqq
\frac{d_l+0.1}{2},
\\
\mathrm{BR}
&\coloneqq
\sum_{l\in\mI}
\frac{
\|W_l\|
}{
\det\left(
\b I+W_l^{\top}W_l
\right)^{1/4}
}.
\end{align*}
Each penalty is added to the empirical training loss with coefficient $\gamma=0.01$. The identity inside the determinant prevents singularity for rectangular or rank-deficient weights. It also makes these penalties different from the determinant terms in the proved injective bound.

Figure~\ref{fig:experiment}(b) reports test accuracy across five independent initializations. In this experiment, the Brownian-inspired penalty produces a higher mean test-accuracy trajectory than the unregularized baseline, whereas the Sobolev-inspired penalty produces a lower trajectory. This is an empirical observation for the stated architecture and hyperparameters, not a general ordering theorem for the two RKHS regimes.

\endgroup

\end{document}